\documentclass{article}

\PassOptionsToPackage{numbers, compress}{natbib}
\usepackage[final]{neurips_2020}
\usepackage[utf8]{inputenc} %
\usepackage[T1]{fontenc}    %
\usepackage{booktabs}       %
\usepackage{amsfonts}       %
\usepackage{nicefrac}       %
\usepackage{microtype}      %

\usepackage{amsmath,amsfonts,bm}

\def\eqref#1{equation~\ref{#1}}

\def\1{\bm{1}}

\def\vx{{\bm{x}}}

\def\vz{{\bm{z}}}

\DeclareMathAlphabet{\mathsfit}{\encodingdefault}{\sfdefault}{m}{sl}
\SetMathAlphabet{\mathsfit}{bold}{\encodingdefault}{\sfdefault}{bx}{n}

\DeclareMathOperator*{\argmax}{arg\,max}

\usepackage[utf8]{inputenc} %
\usepackage{lipsum} %

\usepackage[T1]{fontenc}

\usepackage{etoolbox}
\newbool{includeappendix}
\setbool{includeappendix}{true}

\ifdefined\isoverfull
\else
\fi

\newcommand{\eg}{e.g., }
\newcommand{\ie}{i.e., }

\newcommand{\nth}[1]{$#1^\text{th}$}

\usepackage{acro} %

\DeclareAcronym{cli} {
    short = CLI,
    long = Command Line Interface,
    class = abbrev
}

\usepackage{listings}

\usepackage{textcomp}

\usepackage{xcolor}

\usepackage[scaled=0.8]{beramono}

\definecolor{ckeyword}{HTML}{7F0055}
\definecolor{ccomment}{HTML}{3F7F5F}
\definecolor{cstring}{HTML}{2A0099}

\lstdefinestyle{numbers}{
	numbers=left,
	framexleftmargin=20pt,
	numberstyle=\tiny,
	firstnumber=auto,
	numbersep=1em,
	xleftmargin=2em
}

\lstdefinestyle{layout}{
	frame=none,
	captionpos=b,
}

\lstdefinestyle{comment-style}{
	morecomment=[l]//,
	morecomment=[s]{/*}{*/},
	commentstyle={\color{ccomment}\itshape},
}

\lstdefinestyle{string-style}{
	morestring=[b]",%
	morestring=[b]',%
	stringstyle={\color{cstring}},
	showstringspaces=false,%
}

\lstdefinestyle{keyword-style}{
	keywordstyle={\ttfamily\bfseries},
	morekeywords={
		function,
		constructor,
		int,
		bool,
		return,
		returns,
		uint
	},
	morekeywords = [2]{},
	keywordstyle = [2]{\text},
	sensitive=true,
}

\lstdefinestyle{input-encoding}{
	inputencoding=utf8,
	extendedchars=true,
	literate=
	{ℝ}{$\reals$}1%
	{→}{$\rightarrow$}1%
	{α}{$\alpha$}1%
	{β}{$\beta$}1%
	{λ}{$\lambda$}1%
	{θ}{$\theta$}1%
	{ϕ}{$\phi$}1%
}

\lstdefinestyle{escaping}{
	moredelim={**[is][\color{blue}]{\%}{\%}},
	escapechar=|,
	mathescape=true
}

\lstdefinestyle{default-style}{
	basicstyle=\fontencoding{T1}\ttfamily\footnotesize,
	style=numbers,
	style=layout,
	style=comment-style,
	style=string-style,
	style=keyword-style,
	style=input-encoding,
	style=escaping,
	tabsize=2,
	upquote=true
}

\lstdefinelanguage{BASIC}{
	language=C++,
	style=default-style
}[keywords,comments,strings]%

\usepackage{graphicx}
\usepackage{fixltx2e}
\usepackage{amsthm}
\usepackage{amsmath}
\usepackage{bbold}
\usepackage{enumerate}
\usepackage{algorithm2e,setspace}
\usepackage{caption}
\usepackage{subcaption}
\usepackage{wrapfig}
\usepackage{bbm}
\usepackage{booktabs, multirow}
\usepackage{xspace}
\usepackage{pifont}

\theoremstyle{plain} %
\newtheorem{theorem}{Theorem}[section]
\newtheorem*{theorem*}{Theorem}
\newtheorem{lemma}{Lemma}
\newtheorem*{lemma*}{Lemma}

\newtheorem*{corollary*}{Corollary}

\theoremstyle{definition}

\def\veps{{\mathbf{\epsilon}}}

\newenvironment{psmallmatrix}
{\left(\begin{smallmatrix}}
    {\end{smallmatrix}\right)}

\newcommand{\prob}[0]{\mathbb{P}} %

\makeatletter
\newcommand{\xRightarrow}[2][]{\ext@arrow 0359\Rightarrowfill@{#1}{#2}}
\makeatother

\newcommand{\methodHeuristic}{\textsc{BaseSPT}\xspace}
\newcommand{\methodGlobalD}{\textsc{DistSPT}$^{\mathcal{D}}$\xspace}
\newcommand{\methodGlobalx}{\textsc{DistSPT}$^{\vx}$\xspace}
\newcommand{\methodGlobal}{\textsc{DistSPT}\xspace}
\newcommand{\methodIndividual}{\textsc{IndivSPT}\xspace}

\newcommand{\cmark}{\ding{51}\xspace}%
\newcommand{\xmark}{\ding{55}\xspace}%

\usepackage{hyperref}       %
\usepackage{url}            %

\usepackage[capitalize]{cleveref}

\crefformat{section}{\S#2#1#3}

\crefrangeformat{section}{\S#3#1#4\crefrangeconjunction\S#5#2#6}

\crefmultiformat{section}{\S#2#1#3}{\crefpairconjunction\S#2#1#3}{\crefmiddleconjunction\S#2#1#3}{\creflastconjunction\S#2#1#3}

\newcommand{\crefrangeconjunction}{--}

\crefname{listing}{Lst.}{listings}
\crefname{line}{Lin.}{Lin.}
\crefname{appendix}{App.}{App.}

\newcommand{\app}[1]{%
	\ifbool{includeappendix}{\cref{#1}}{the appendix}%
}
\newcommand{\App}[1]{%
	\ifbool{includeappendix}{\cref{#1}}{The appendix}%
}

\usepackage[capitalize]{cleveref}
\crefformat{section}{Section~#2#1#3}
\crefname{algocf}{Algorithm}{Algorithms}
\crefformat{footnote}{#2\footnotemark[#1]#3}

\title{Certified Defense to Image Transformations via Randomized Smoothing}
\author{
	Marc Fischer, Maximilian Baader, Martin Vechev \\
	Department of Computer Science\\
	ETH Zurich\\
	\texttt{\{marc.fischer, mbaader, martin.vechev\}@inf.ethz.ch} \\
}

\begin{document}

\maketitle

\begin{abstract}
  We extend randomized smoothing to cover parameterized transformations (e.g., rotations, translations) and certify robustness in the parameter space (e.g., rotation angle). This is particularly challenging as interpolation and rounding effects mean that image transformations do not compose, in turn preventing direct certification of the perturbed image (unlike certification with $\ell^p$ norms). We address this challenge by introducing three different kinds of defenses, each with a different guarantee (heuristic, distributional and individual) stemming from the method used to bound the interpolation error. Importantly, we show  how individual certificates can be obtained via either statistical error bounds or efficient online inverse computation of the image transformation. We provide an implementation of all methods at \url{https://github.com/eth-sri/transformation-smoothing}.

\end{abstract}

\section{Introduction}

Deep neural networks are vulnerable to adversarial examples \cite{SzegedyAdvExamples} -- small changes that preserve semantics (\eg $\ell^p$-noise or geometric transformations such as rotations) \citep{EngstromMadryRotation}, but can affect the output of a network in undesirable ways. As a result, there has been substantial recent interest in methods which aim to ensure the network is certifiably robust to adversarial examples \cite{AI2,DiffAI,WongK18,RaghunathanSL18a,CohenRK19,Salman, PeiCYJ17,  DeepPoly, DeepG, MohapatraGeometric, LinyiLiSmoothing}.

\paragraph{Certification guarantees}
There are two principal robustness guarantees a certified defense can provide at inference time:
(i) the (standard) distributional guarantee, where a  robustness score is
computed offline on the test set to be interpreted in expectation for
images drawn from the data distribution, and (ii) an individual guarantee, where a certificate is
computed online for the (possibly perturbed) input. The choice of guarantee depends on the application and regulatory constraints.

\paragraph{Guarantees with $\ell^{p}$ norms}
When considering $\ell^p$ norms, existing certification methods can be directly used to obtain either of the above two guarantees:
for an image $\vx$ and adversarial noise $\delta$, $\|\delta\|_p < r$, proving that a classifier $f$ is $r$-robust around  $\vx':=\vx+\delta$ is enough to guarantee $f(\vx) = f(\vx')$.
That is, it suffices to prove robustness of a perturbed input in order to certify that the perturbation did not change the classification, as the $r$-ball around $\vx'$ includes $\vx$.

\paragraph{Key challenge: guarantees for geometric perturbations}
Perhaps not intuitively, however, for more complex perturbations such as geometric transformations, proving robustness around an image $\vx'$ via existing methods (e.g., \cite{PeiCYJ17,DeepPoly,DeepG,MohapatraGeometric}) does not imply that
$f(\vx) = f(\vx')$ for the original image $\vx$. To illustrate this issue, consider the rotation $R_\gamma$, by angle $\gamma$ of an image $\vx$, followed by an interpolation $I$. Certifying that the classification of the rotated image $\vx' := I \circ R_\gamma(\vx)$ for $\|\gamma\|<r$ is robust under further rotations $I \circ R_\beta$ for $\|\beta\|<r$ is not sufficient to imply that $\vx$ and $\vx'$ classify the same, as rotating $\vx'$ back by $\beta = -\gamma$ does not return the original image $\vx$ due to interpolation.
A central challenge then is to develop techniques that are able to handle more involved perturbations.

\paragraph{This work: certification beyond $\ell^{p}$ norms}
\begin{wrapfigure}[15]{r}{0.5\textwidth}
  \footnotesize
  \centering
  \captionof{table}{Certificates obtained by different methods. $^{*}$ indicates deterministic certification, other methods hold with high confidence.} \label{tab:overview}
  \begin{tabular}[t]{@{}lcc@{}}
    \toprule
    & dist. & indiv.\\
    \midrule
    \multicolumn{3}{l}{Composable perturbation $\psi$ (e.g., additive $\ell^{p}$-bound)}\\
    \midrule
    relaxation-based$^{*}$ \cite{AI2,DiffAI,WongK18,RaghunathanSL18a}
    & \cmark & \cmark \\
    \citet{CohenRK19}
    & \cmark & \cmark \\
    \midrule
    \multicolumn{3}{l}{Non-composable $\psi$ (e.g., rotation $I \circ R$)}\\
    \midrule
    \methodIndividual & (\cmark) & \cmark \\
    \methodGlobalD & \cmark w.p. $q_{E}$ & \cmark w.p. $q_{E}$\\
    \methodGlobalx & \cmark  & \xmark \\
    relaxation-based$^{*}$  \cite{PeiCYJ17, DeepPoly, DeepG, MohapatraGeometric} & \cmark & \xmark\\
    RS-based \cite{LinyiLiSmoothing,TSS} & \cmark & \xmark \\
\bottomrule
  \end{tabular}
\end{wrapfigure}
In this work we address this challenge and introduce the first certification methods for geometric transformations based on randomized smoothing (RS): we extend RS \cite{CohenRK19} to handle transformations (SPT) by adding (Gaussian) noise to transformation parameters, enabling us to handle large models and datasets (\eg ImageNet). Our methods, their guarantees and how they compare to standard RS  \cite{CohenRK19} (on $\ell^{p}$ norms) and other techniques, are shown in \cref{tab:overview}.

\paragraph{\methodHeuristic} As with standard RS over $\ell^{p}$ norms,  SPT (not listed) provides individual and distributional guarantees, but only for composable parametric transformations, that is, where: $\psi_\gamma$: $\psi_{\beta + \gamma} = \psi_\gamma \circ \psi_\beta$. For non-composable ones (e.g., rotations with interpolation), \methodHeuristic is only a heuristic defense, motivating the need for appropriate certification methods.

\paragraph{\methodIndividual} This method provides the strongest guarantees for non-composable transformations and works as follows: at inference time, for each input $\vx'$, it calculates an individual upper bound of the expression $\epsilon$ \emph{without access} to (the original) $\vx$, then combined with SPT and smoothing. A key step here is computing the inverse $\psi_\gamma^{-1}(\vx')$ of $\vx'$, for which we introduce an efficient technique.

\paragraph{\methodGlobal} While desirable (it mimics original RS guarantees), \methodIndividual can be expensive to apply at inference time and obtain tight certificates with. This motivates the study of more relaxed, still useful certification guarantees, as well as corresponding methods which achieve tighter bounds using these definitions. The idea of \methodGlobal is to estimate a probabilistic upper bound for the expression $\epsilon = \|\psi_\beta \circ \psi_\gamma(\vx)  - \psi_{\beta + \gamma}(\vx)\|_2$, combined with SPT and RS. The first variant here is \methodGlobalD, where this upper bound is estimated offline on the training dataset and holds for all $\vx$ from the data distribution $\mathcal{D}$, with probability $q_{E}$. This method enjoys both probabilistic distributional and individual guarantees. The weakening of the definition used by \methodIndividual (now probabilistic over $q_E$) enables the method to compute tighter bounds. The second variant, \methodGlobalx, provides weaker guarantees than \methodGlobalD, with the provided bound now computed for individual $\vx$ on the test set. It obtains a distributional guarantee, however, it does not provide individual guarantees -- this restriction allows \methodGlobalx to compute even tighter bounds. We remark that recent methods targeting robustness to geometric transformations (e.g., \cite{DeepG, LinyiLiSmoothing,TSS} also fall in this class.

To summarize, our core contributions are:
\begin{itemize}
  \item A generalization of randomized smoothing to parameterized transformations.
  \item A number of novel certification methods for non-composable parameterized transformations, systematically exploring both distributional and individual guarantees while considering deterministic and probabilistic bounds. In the process, we highlight the rich interplay between certification definitions and tightness of the corresponding certificates.
\item A thorough evaluation of all methods on common image datasets, showcasing certified robustness to $\pm 30^{\circ}$ rotations for $50\%$ of inputs on Restricted ImageNet.
\end{itemize}

\section{Related Work}

We now survey the most closely related work in neural network certification and defenses.

\paragraph{$\ell^p$ norm based certification and defenses}
The discovery of adversarial examples \citep{SzegedyAdvExamples,
BiggioAdvExamples} triggered interest in training and certifying robust neural
networks.
An attempt to improve model robustness are empirical defenses
\cite{CaoRegionbased, LiuSmoothing}, strategies which harden a model against an
adversary.
While this may improve robustness to current adversaries, typically robustness
cannot be formally verified with current certification methods. This is
because complete methods \citep{Ehlers17planet, Reluplex, bunel18nips} do not
scale and incomplete methods relying on over approximation lose too much
precision \cite{AI2, WangSafety, Weng2018,
RaghunathanSL18a, DeepPoly, SalmanBarrier}, even for networks trained
to be amenable to certification.
Recently, randomized smoothing was introduced, which could for the first time,
certify a (smoothed) classifier against norm bound $\ell^2$ noise on ImageNet
\cite{Lecuyer2018CertifiedRT, LiSampling, CohenRK19, Salman, Macer}, by relaxing exact
certificates to high confidence probabilistic ones.
Smoothing scales to large models, however, it is currently limited to norm-based perturbations.

\paragraph{Semantic perturbations}
Transformations such as translations and rotations can produce adversarial examples \cite{EngstromMadryRotation, KanbakMF18}.
An enumerative approach certifying against semantic perturbations was presented
in \cite{PeiCYJ17}. There, the search space is reduced by only consider next
neighbor interpolation. Unfortunately, for more elaborate interpolations (e.g., bilinear), the approach becomes infeasible. The first
certification against rotations with bilinear interpolations was carried out in \cite{DeepPoly}, later significantly improved on by \cite{DeepG}. Both
methods generate linear relaxations and propagate them through the network. However, the methods do not yet scale to large networks (\ie ResNet-50) or
complex data sets (\ie ImageNet). The approaches of \cite{MohapatraGeometric} and \cite{DeepPoly} are similar for rotation. \cite{LinyiLiSmoothing,TSS} reduce transformations to multiple $\ell_{2}$-balls which they certify via RS so to obtain a certificate for the overall transformation. As outlined in \cref{tab:overview}, all these methods result in a distributional but not an individual guarantee.

\section{Generalization of Smoothing} \label{sec:generalization}

A smoothed classifier $g \colon \mathbb{R}^m \mapsto \mathcal{Y}$ can be
constructed out of an ordinary classifier $f \colon \mathbb{R}^m \mapsto
\mathcal{Y}$, by calculating the most probable result of $f(\vx + \epsilon)$ where
$\epsilon \sim \mathcal{N}(0, \sigma^2 \mathbb{1})$:
\begin{equation*}
    g(\vx) := \argmax_c \prob_{\epsilon \sim \mathcal{N}
        (0, \sigma^2 \mathbb{1})}(f(\vx + \epsilon)=c).
\end{equation*}
One then obtains the following robustness guarantee:
\begin{theorem}[From \cite{CohenRK19}] \label{thm:original}
    Suppose $c_A \in \mathcal{Y}$, $\underline{p_A}, \overline{p_B} \in [0,1]$. If
    \begin{equation*}
        \prob_{\epsilon}(f(\vx + \epsilon)=c_A)
        \geq
        \underline{p_A}
        \geq
        \overline{p_B}
        \geq
        \max_{c \neq c_A}\prob_{\epsilon}(f(\vx + \epsilon)=c),
    \end{equation*}
    then $g(\vx + \delta) = c_A$ for all $\delta$ satisfying $\|\delta\|_2 \leq
    \tfrac{\sigma}{2}(\Phi^{-1}(\underline{p_A}) - \Phi^{-1}(\overline{p_B})) =:
    r_{\delta}$.
\end{theorem}

We now generalize this theorem to parameterized transformations. Consider the composable
transformations $\psi_\beta: \mathbb{R}^m \to \mathbb{R}^m$, satisfying
$\psi_\beta \circ \psi_\gamma = \psi_{\beta + \gamma}$ for all $\beta, \gamma \in
\mathbb{R}^d$. Then we can define a smoothed classifier $g: \mathbb{R}^m \to
\mathcal{Y}$ analogously for a parametric transformation $\psi_\beta$ by
\begin{equation}
    \label{eq:g}
  g(\vx) = \argmax_c
    \prob_{\beta~\sim~\mathcal{N}(0, \sigma^2 \mathbb{1})}
    \left(
      f \circ \psi_\beta(\vx) = c
    \right).
\end{equation}
With that, we obtain the following robustness guarantee:
\begin{theorem} \label{thm:bound}
  Let $\vx \in \mathbb{R}^m$, $f: \mathbb{R}^m \to \mathcal{Y}$ be a classifier
  and $\psi_\beta: \mathbb{R}^m \to \mathbb{R}^m$ be a composable transformation
  as above. If
  \begin{equation*}
    \prob_\beta(f \circ \psi_\beta(\vx) = c_{A})
    \geq
    \underline{p_{A}}
    \geq
    \overline{p_{B}}
    \geq
    \max_{c_{B} \neq c_{A}} \prob_\beta(f \circ \psi_\beta(\vx) = c_{B}),
  \end{equation*}
  then $g \circ \psi_\gamma(\vx) = c_A$ for all $\gamma$ satisfying
  $ \|\gamma\|_2 \leq \tfrac{\sigma}{2}(\Phi^{-1}(\underline{p_{A}}) -
  \Phi^{-1}(\overline{p_{B}})) =: r_\gamma.  $ Further, if $g$ is
  evaluated on a proxy classifier $f'$ that behaves like $f$ with
  probability $1-\rho$ and else returns an arbitrary answer, then
  $r_\gamma := \tfrac{\sigma}{2}(\Phi^{-1}(\underline{p_{A}}-\rho) -
  \Phi^{-1}(\overline{p_{B}}+\rho))$.
\end{theorem}

The proof is similar to the one presented in \citet{CohenRK19} and is given in
\cref{proof}. The key difference is that we allow
parameterized transformations $\psi$, while \citet{CohenRK19} only allows
additive noise.

\section{Certification with interpolation and rounding errors}\label{sec:approach} %

\begin{wrapfigure}{r}{.3\textwidth}
  \vspace{-0.5em}
  \centering
  \includegraphics[width=0.3\textwidth]{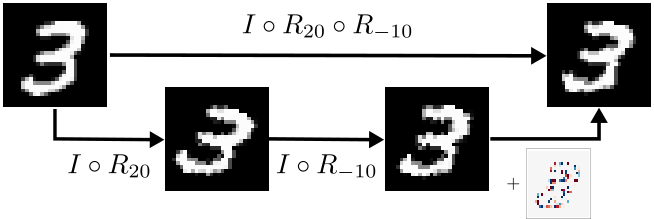}
  \caption{Rotations with interpolation do not compose.}
  \label{fig:interpolation}
  \vspace{-1.5em}
\end{wrapfigure}

We now instantiate \cref{thm:bound} for parameterized geometric image
transformations $T_\beta$, $\beta \in \mathbb{R}^d$, followed by interpolation $I$, denoted as $T_\beta^I$.
A geometric transformation $T_\beta$ is followed by an interpolation $I$ in order to express the result on the pixel grid.
In general, even if $T_{\beta}$ composes, $T_{\beta}^{I}$ does not (see \cref{fig:interpolation} in the case where $T_\beta$ is a rotation $R_\beta$ by an angle $\beta$).
This prevents us from directly instantiating \cref{thm:bound} with $\psi_{\beta} := T_{\beta}^{I}$.

To address this issue, we now show how to construct a classifier $g_E$ with the desired guarantees, namely that
$g_E \circ T^{I}_{\gamma}(\vx) = g_{E}(\vx)$ for $\gamma$ with $\|\gamma\|_{2} \leq r_{\gamma}$, thus enabling certification of image transformations (which may not compose).
Our proposed construction consists of two steps.

First, for a fixed but arbitrary $\vx$, let $h_{E}$ be a classifier satisfying interpolation invariance:
\begin{equation}
  h_{E} \circ T^{I}_{\beta} \circ T^{I}_{\gamma} (\vx) = h_{E} \circ
  T^{I}_{\beta + \gamma}(\vx)
  \;\;
  \forall \beta, \gamma \in \mathbb{R}^d.
  \label{eq:h_E:cond}
\end{equation}
We now instantiate \cref{thm:bound} with $f := h_{E} \circ I$ and $\psi_{\beta} := T_{\beta}$, obtaining
a smoothed classifier
$
  g_E(\vx) := \argmax_c
    \prob_{\beta~\sim~\mathcal{N}(0, \sigma^2 \mathbb{1})}
    \left(
      h_{E} \circ I \circ T_\beta(\vx) = c
    \right),
$
such that $g_E \circ T_{\gamma}(\vx) = c_{A} = g_E(\vx)$ for $\gamma$ with $\|\gamma\|_{2} \leq r_{\gamma}$ by \cref{thm:bound}. Further, since
\begin{align*}
  g_E \circ T_{\gamma}(\vx)
  &=
  \argmax_c
    \prob_{\beta~\sim~\mathcal{N}(0, \sigma^2 \mathbb{1})}
    \left(
      h_{E} \circ I \circ T_\beta \circ T_{\gamma} (\vx) = c
    \right)
  \\
  &=
  \argmax_c
    \prob_{\beta~\sim~\mathcal{N}(0, \sigma^2 \mathbb{1})}
    \left(
      h_{E} \circ T^{I}_\beta \circ T^{I}_{\gamma} (\vx) = c
    \right)
  \\
  &= g_E \circ T^{I}_{\gamma}(\vx),
  \end{align*}

where the first and last equities hold by the definition of $g_E$ and the second one due to \cref{eq:h_E:cond}. Thus, we obtain a classifier $g_E$ with the desired property.

Second, we discuss the construction of the desired $h_{E}$ (from step 1). Consider the interpolation error
\begin{align}
  &\veps(\beta, \gamma, \vx)
  := T_\beta^I \circ
  T_\gamma^I(\vx) - T_{\beta + \gamma}^I(\vx)   \label{eq:epsilon_error},\\
  &\text{bounded by $E \in \mathbb{R}^{\geq 0}$ s.t.}\;
  \forall \beta, \gamma \in \mathbb{R}^d. \; \|\veps(\beta, \gamma, \vx)\|_{2} \leq E \label{eq:condE}
\end{align}
for a given but arbitrary $\vx$. Thus if $h_{E}$ is $\ell^{2}$-robust
with radius $E$ around $T_{\beta + \gamma}^I(\vx)$, interpolation
invariance holds.
While many choices for such $h_{E}$ are possible in the rest of the paper we instantiate $h_{E}$ by applying \cref{thm:original} to a
base classifier $b$.

\paragraph{Obtaining probabilistic guarantees from  \cref{thm:bound}} So far we assumed that $\vx$ is arbitrary but fixed and constructed $E$ and $h_E$ for this $\vx$ specifically.
In general, finding a tight deterministic bound $E$ that holds $\forall \beta, \gamma$ is computationally challenging. Thus, we relax this deterministic guarantee into a probabilistic one:
\begin{equation}
  \mathbb{P}_{\beta \sim \mathcal{N}(0, \sigma^2\mathbb{1})}
  \left(
  \|\veps(\beta, \gamma, \vx)\|_{2} \leq E
  \right) \geq 1-\rho_E
  \;\;
  \forall \gamma \in \mathbb{R}^d.
  \label{eq:relaxed:condE}
\end{equation}
meaning \cref{eq:condE} holds with probability at least $1-\rho_E$, in turn implying that \cref{eq:h_E:cond} also holds at least with probability $1-\rho_E$.
This can also be formulated as having a proxy classifier $h'_{E}$ which behaves like $h_{E}$ with probability at least $1-\rho_E$ on the inputs specified by \cref{eq:h_E:cond}.
In practice, we construct $h'_{E}$ which behaves like $h_{E}$ with probability at least $1-\rho_E$ on all inputs, implying this behavior on the inputs from \cref{eq:h_E:cond}.
From $h'_{E}$, we then obtain $f' := h'_{E} \circ I$ which behaves like $f$ with probability at least $1-\rho_E$ on all inputs.
Then, we can apply \cref{thm:bound} by setting $\rho$ to $\rho_E$ and obtain the desired guarantee.
In \cref{sec:calc_error_bounds}, we show how to obtain $E$ for \methodGlobal and \methodIndividual.

\section{Calculation of error bounds} \label{sec:calc_error_bounds}
In \cref{sec:approach:distributional_bounds} we derive a
distributional error bound over a dataset and in
\cref{sec:approach:individual_bounds} a per-image bound. Throughout
this section, we assume the attacker model
$\gamma \in \Gamma \subseteq \mathbb{R}^d$. As we compute $E$ with
this assumption, our obtained certificate proves robustness of $g_{E}$
to $T^{I}_{\gamma}$ for $\gamma \in \Gamma$ with
$\|\gamma\|_{2} \leq r_{\gamma}$.

\subsection{Distributional bounds for \methodGlobal} \label{sec:approach:distributional_bounds}
For a fixed $E \in \mathbb{R}^{\geq 0}, \rho_E \in [0, 1]$, the probability
that $\epsilon$ is bounded by $E$ for $\vx \sim \mathcal{D}$ is
\begin{equation} \label{eq:qE}
  q_E
  :=
  \mathbb{P}_{\vx \sim \mathcal{D}}
  (
    \mathbb{P}_{\beta \sim \mathcal{N}(0, \sigma^2 \mathbb{1})}
      (
      \max_{\gamma \in \Gamma}
      \|\epsilon(\beta, \gamma, \vx)\|_2
      \leq
      E\
      ) \geq 1 - \rho_E
    ).
\end{equation}
In practice, for \methodGlobalD we evaluate $q_{E}$ by sampling $\vx$ and counting how
often the inner property holds. We compute the inner probability by: (i)
sampling multiple realizations of $\beta$, (ii) computing their
corresponding error $\epsilon$ and checking how many are successfully
bounded by $E$, and (iii) bounding the inner probability using
Clopper-Pearson. If this lower bound is larger than $1 - \rho_E$
we count this as a positive sample, else a negative one. Once these counts are obtained for a number
of sampled points $\vx$, we can apply Clopper-Pearson and obtain a lower bound $\underline{q_{E}}$ with the desired confidence.
For \methodGlobalx only the inner probability needs to be computed for an individual image $\vx$.
Formally this can be seen as considering the data distribution $\mathcal{D}$ that just contains $\vx$ (thus $q_{E} = 1$).

To compute the maximization over $\gamma$ we employ standard interval
analysis, which allows us to efficiently propagate lower and upper bounds \cite{Intervals}.
By propagating the hyperrectangle containing $\Gamma$ along with the sampled $\beta$ and
$\vx$, we eventually obtain a lower and upper bound for the norm calculation of which we take the maximum:
\begin{equation}
  \max_{\gamma \in \Gamma}\| \epsilon(\beta, \gamma, \vx) \|_2
  \leq
  \max
  \| T_\beta^I \circ T_{\Gamma}^I(\vx) - T_{\beta + \Gamma}^I(\vx) \|_2.
  \label{eq:maxEps}
\end{equation}
The result can be refined by splitting the hyperrectangle $\Gamma$ into smaller hyperrectangles $\Gamma_k$ for $k \in \{0, \dots, N\}$. The refined bound is
\begin{equation}
  \label{eq:disterr:bound}
  \max_{\gamma \in \Gamma}\| \epsilon(\beta, \gamma, \vx) \|_2
  \leq
  \max_{k \in \{0, \dots, N\}}
  \max
  \| T_\beta^I \circ T_{\Gamma_k}^I(\vx) - T_{\beta + \Gamma_k}^I(\vx) \|_2.
\end{equation}
To obtain $E$ in the first place, we perform the same sampling
operations as above (sample $\vx$ and $\beta$) but do not compute any probabilities,
that is, for each sample ($\vx$, $\beta$), we simply keep the values attained by \cref{eq:disterr:bound}.

For \methodGlobalD we pick an $E$ that bounds many of these values, choosing $\rho_E$ to be small.
Once $E$ is obtained, we compute $q_{E}$ as described above.
Instantiating the construction of \cref{sec:approach} with this $E$ yields the guarantee that for a
random image $\vx \sim \mathcal{D}$ the guarantees provided by  \cref{thm:bound} hold with probability $q_E$.

For \methodGlobalx, after we determine a suitable $E$ for the given $\vx$ we can determine $\rho_{E}$.

\subsection{Individual bounds for
  \methodIndividual} \label{sec:approach:individual_bounds} %
At inference time, we are given $\vx' := T^{I}_\gamma(\vx)$ but neither the original $\vx$ nor the parameter $\gamma \in \Gamma$, and we would like to certify that $g_E(\vx') = g_E(\vx)$.
When $\psi_{\beta}$ composes as required in \cref{sec:generalization}, this can be certified by showing $g$ is robust with a sufficient radius $r_{\gamma}$.
However, when $\psi_{\beta}$ does not compose, this can be accomplished by applying \cref{thm:bound} to show $g_E(\vx)$ is robust with radius $r_\gamma$ that includes $\Gamma$.
In turn, this requires a bound $E$ (see \cref{eq:relaxed:condE}) for $\vx$ (rather than $\vx'$):
\begin{equation}
  \mathbb{P}_{\beta \sim \mathcal{N}(0, \sigma^2 \mathbb{1})}
  \left(
    \max_{\gamma \in \Gamma} \| \epsilon(\beta, \gamma, \vx)  \|_2
    \leq
    E
  \right)
\geq 1 - \rho_E.
  \label{eq:indiv_err_bound}
\end{equation}
Now, we would like to compute an upper bound on the $\max$ term
\emph{without} having access to $\vx$. This is accomplished as
follows: First, in the above equation, we replace $\epsilon$ by its
definition (\cref{eq:epsilon_error}) and $T_\gamma^I(\vx)$ by
$\vx'$. We then replace $\vx$ with a symbolic set of possible inputs
that could have generated $\vx'$, denoted as
$(T_\Gamma^I)^{-1}(\vx') := \{\vx \in \mathbb{R}^{m} \mid
T^{I}_{\gamma}(\vx) = \vx', \gamma \in \Gamma \}$ which we can use
instead of $\vx$ due to the maximization over $\gamma$. As in
\cref{sec:approach:distributional_bounds}, we obtain the resulting bound via interval
analysis:
\begin{equation}
  \max_{\gamma \in \Gamma} \|\epsilon(\beta, \gamma, \vx)\|_2
  \leq
  \max \|
    T_\beta^I(\vx')
    -
    T_{\beta + \Gamma}^I \circ (T_\Gamma^I)^{-1}(\vx')
  \|_2.
  \label{eq:indiv_err_interval}
\end{equation}
The computation of the inverse $(T_\Gamma^I)^{-1}(\vx')$ is explained in \cref{subsec:inverse}.
By substituting \cref{eq:indiv_err_interval} in \cref{eq:indiv_err_bound} we can obtain and verify $E$ as in \cref{sec:approach:distributional_bounds} (except we do not need to sample $\vx$'s).
As before, we can refine the upper bound of \cref{eq:indiv_err_interval} by splitting $\Gamma$ into $\Gamma_k$. We note as the inverse does not depend on $\beta$, given $\vx'$, it only needs to be computed once
and can be reused whenever we evaluate
  \cref{eq:indiv_err_interval} for a given sample $\beta$.

\section{Inverse Computation} \label{subsec:inverse}
\begin{figure*}%
    \centering
    \begin{subfigure}[t]{.15\textwidth}
        \centering
        \includegraphics[width=0.7\textwidth]{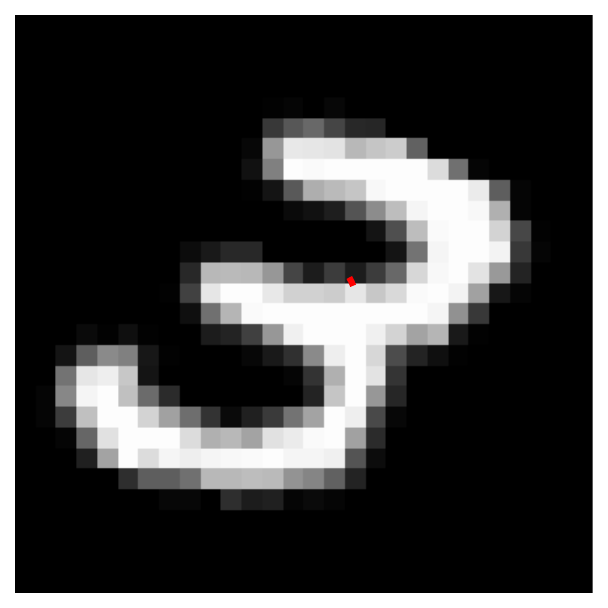}
        \subcaption{Rotated}
        \label{fig:Inverse_Calc:rotated}
    \end{subfigure}
    \begin{subfigure}[t]{.30\textwidth}
        \centering
        \includegraphics[width=0.35\textwidth]{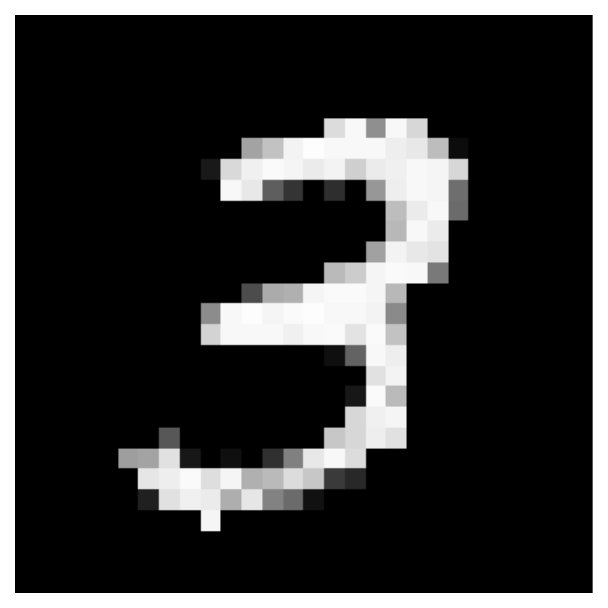}
        \includegraphics[width=0.35\textwidth]{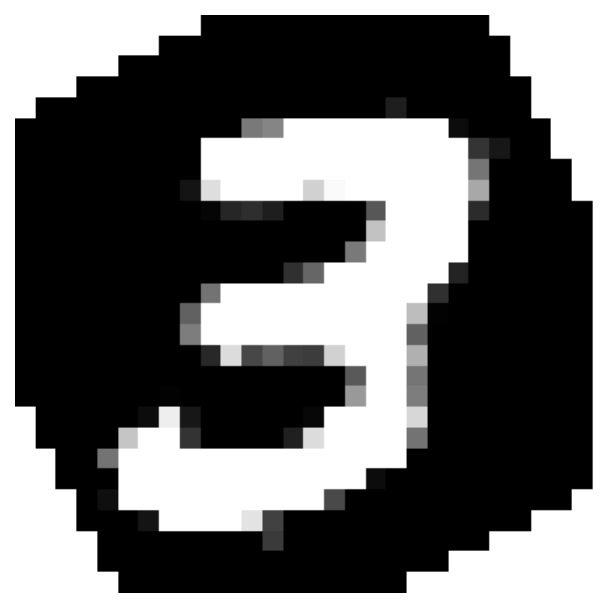}
        \subcaption{Inverse}
        \label{fig:Inverse_Calc:inverse}
    \end{subfigure}
    \begin{subfigure}[t]{.30\textwidth}
        \centering
        \includegraphics[width=0.35\textwidth]{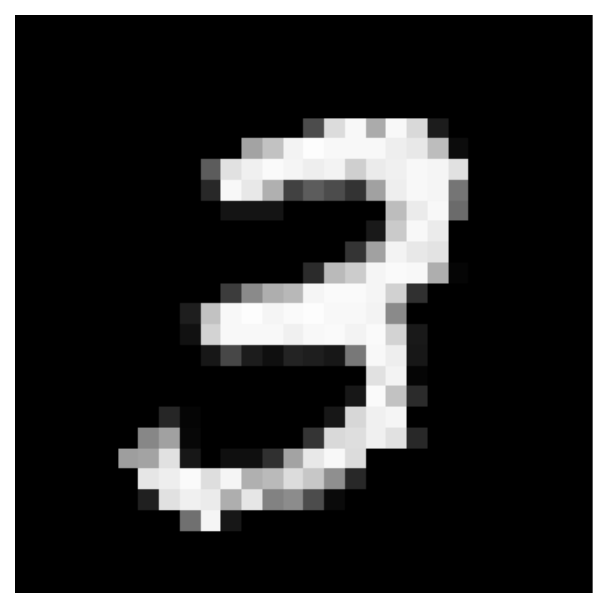}
        \includegraphics[width=0.35\textwidth]{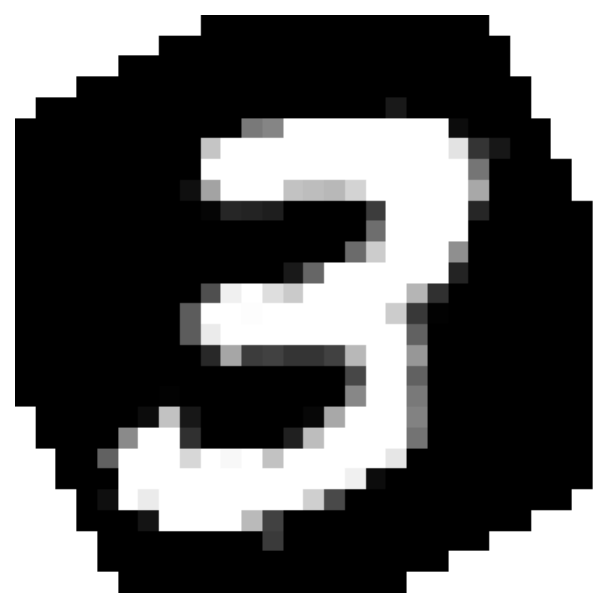}
        \subcaption{$10\times$ refined inverse}
        \label{fig:Inverse_Calc:refined}
    \end{subfigure}
    \begin{subfigure}[t]{.15\textwidth}
        \centering
        \includegraphics[width=0.7\textwidth]{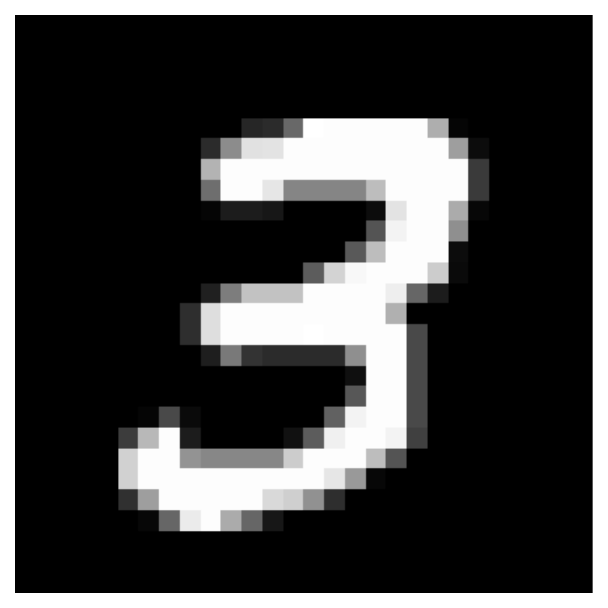}
        \subcaption{Original}
        \label{fig:Inverse_Calc:original}
    \end{subfigure}
    \caption{Over approximation of the inverse image. 
    The image pairs (b) and (c) depict the lower (left) and upper (right) interval pixel bounds for the
    inverse image and the $10\times$ refined image respectively.}
    \label{fig:Inverse_Calc}
    \vspace{-1em}
\end{figure*}

We now discuss how to obtain a set containing all possible inverse images.
That is, given $\vx' := T^{I}_{\gamma}(\vx)$ and $\gamma \in \Gamma$, we compute the set
$(T^{I}_{\Gamma})^{-1}(\vx')$ which contains all possible $\vx$.
First, we cover the necessary background. To ease
presentation, we assume even image height and width. 
We embed the images in $\mathbb{R}^2$
by centering them at 0 on an odd integer grid
$G := (2 \mathbb{Z}+1) \times (2 \mathbb{Z}+1)$
and centered at 0. We denote the value of a pixel at
$(i,j) \in G$ by $p_{i,j} \in [0, 1]$.

\paragraph{Transformations}
The pixel values $p'_{i',j'}$ for $(i',j') \in G$ of an image, produced by a
transformation $T_\gamma \colon \mathbb{R}^2 \to \mathbb{R}^2$ with parameter
$\gamma \in \mathbb{R}^d$, is calculated by interpolating at the inversely
transformed coordinate $T^{-1}_\gamma(i',j')$, followed by the interpolation $I$
resulting in $p'_{i',j'} = I \circ T^{-1}_\gamma(i',j')$.

\paragraph{Bilinear interpolation}
A prominent interpolation is
\emph{bilinear interpolation}, given by
\begin{equation}
    I(x,y) = p_{v,w}\tfrac{2+v-x}{2}\tfrac{2+w-y}{2}
    + p_{v,w+2}\tfrac{2+v-x}{2}\tfrac{y-w}{2}
    + p_{v+2,w}\tfrac{x-v}{2}\tfrac{2+w-y}{2}
    + p_{v+2,w+2}\tfrac{x-v}{2}\tfrac{y-w}{2},
    \label{eq:Interpolation}
\end{equation}
where $(v, w) \in G$ is the coordinate such that $(x,y)$ lies in the
$(v,w)$-interpolation region, that is $(x,y) \in [v, v+2] \times [w, w+2]$. We use $v$ and $w$ as grid indices in the context of the
interpolation $I$. If $p_{v,w}$ has no defined value because $(v,w)$ is out of
range for the image, we set $p_{v,w}$ to 0. 

We start by giving a procedure to
calculate constraints of a single pixel $(i,j)$ for a single color
channel, after which we present an iterative procedure to refine that
constraint. The inverse image is then obtained by following this
procedure for every pixel in every color channel. We illustrate the
steps in \cref{sec:approach:example} using the example of a rotated
image $\vx'$ (\cref{fig:Inverse_Calc:rotated}).

The attacker transformed the original image $\vx$
(\cref{fig:Inverse_Calc:original}) using $T^{I}_{\gamma}$
for $\gamma \in \Gamma$ and
therefore obtained the pixel values $p'_{i',j'}$ of the transformed image $\vx'$
by evaluating $p'_{i',j'} = I \circ T^{-1}_\gamma(i',j')$. The interpolation $I$
uses the pixel values $p_{i,j}$ of $\vx$. The following steps
invert this relation for every coordinate $(i,j)$:

\paragraph{Step 1} For every $(i',j') \in G$, we over-approximate the region the
pixel value $p'_{i',j'}$ could have been interpolated from, which is
$c_{i',j'} := T^{-1}_{\Gamma}(i',j')$,
$C := \{ c_{i',j'} \mid (i',j') \in G \}$.
In practice, only a finite subset of $C$ is used. In \cref{sec:refimenet},
we show how to calculate this subset efficiently.

\paragraph{Step 2} The interpolation $I$ is defined piecewise per
$(v,w)$-interpolation region $[v, v+2] \times [w, w+2]$, so the
algebraic form of $I$, \cref{eq:Interpolation} holds for each
interpolation region separately.  For every interpolation region
cornering $(i,j)$ that $c_{i',j'}$ intersects with, the pixel value
$p'_{i',j'}$ yields constraints for value $p_{i,j}$.  Here, we
describe just the constraint $q_{i,j}$ associated with
the $(i,j)$-interpolation region; others
($(i-2,j-2), (i-2,j), (i,j-2)$) work analogously. First, for every
$c_{i',j'} \in C$ we calculate its intersection with the
$(i,j)$-interpolation region, yielding
\begin{equation*}
  [x_l, x_u] \times [y_l,y_u] := c_{i',j'} \cap [i,i+2] \times [j,j+2].
\end{equation*}
We can plug this into the interpolation $I$, where we instantiate $(v,w) \gets
(i,j)$, resulting into %
\begin{equation}
  \begin{aligned}
    p'_{i',j'}
    \in
    I([x_l,x_u],[y_l,y_u])
    = \;
    & p_{i,j}\tfrac{2+i-[x_l,x_u]}{2}\tfrac{2+j-[y_l,y_u]}{2}
    + p_{i,j+2}\tfrac{2+i-[x_l,x_u]}{2}\tfrac{[y_l,y_u]-j}{2}
    \\
    &+ p_{i+2,j}\tfrac{[x_l,x_u]-i}{2}\tfrac{2+j-[y_l,y_u]}{2}
    + p_{i+2,j+2}\tfrac{[x_l,x_u]-i}{2}\tfrac{[y_l,y_u]-j}{2}.
  \end{aligned}
  \label{eq:constraints_pre}
\end{equation}
Next, we solve for the pixel value of interest $p_{i,j}$. Then, we replace
all other three pixel values $p_{i,j+2}$, $p_{i+2,j}$, and $p_{i+2,j+2}$
with the (trivial) $[0,1]$ constraint, covering all possible pixel
values. While this results into sound constraints for
$p_{i,j}$, instantiating $[x_l,x_u]$ and $[y_l,y_u]$ with its
corner $(x,y)$ furthest from $(i,j)$, yields still a sound but
more precise constraint $q_{i,j}$ for $p_{i,j}$. Here, this
amounts to $x \gets x_u$ and $y \gets y_u$.
\cref{sec:refimenet} presents a detailed explanation of the derivation.
The result is
\begin{align*}
  q_{i,j}
  =
  \left[
      p'_{i',j'} - \left(\tfrac{2+i-x_u}{2}\tfrac{y_u-j}{2}
      + \tfrac{x_u-i}{2}\tfrac{2+j-y_u}{2}
      + \tfrac{x_u-i}{2}\tfrac{y_u-j}{2}\right)
      ,
      p'_{i',j'}
  \right]
  \left(\tfrac{2+i-x_u}{2}\tfrac{2+j-y_u}{2}\right)^{-1}.
\end{align*}

\paragraph{Step 3} In order to be sound, we need to take the
union over $q_{i-2,j-2}, q_{i-2,j}, q_{i,j-2}, q_{i,j}$
for each $c_{i',j'}$. To gain precision, we can intersect all of those unions
and finally, we can intersect this constraint with the trivial
one, $[0,1]$, resulting in the final pixel constraint for pixel $p_{i,j}$:
\begin{equation}
  p_{i,j} \in [0,1] \cap
  \Big(
  \bigcap\limits_{c_{i',j'} \in C}
    q_{i-2,j-2}(c_{i',j'})
    \sqcup q_{i,j-2}(c_{i', j'})
    \sqcup q_{i-2,j}(c_{i', j'})
    \sqcup q_{i,j}(c_{i', j'})
  \Big),
  \label{eq:constraints_final}
\end{equation}
where $\sqcup$ denotes the \emph{join} operation, that is $[a,b] \sqcup [c,d] :=
[\min(a,c), \max(b,d)]$. If the intersection of $c_{i',j'}$ with the respective
$(v,w)$-interpolation region is empty, we omit $q_{v,w}$ in
\cref{eq:constraints_final}.

In \cref{sec:approach:individual_bounds}, we split $\Gamma$ into $\Gamma_k$.
It often happens that
one of the resulting
intervals is empty. Then we know for sure that $\gamma$ lies in a different $\Gamma_k$, speeding up the process substantially.

\paragraph{Refined Inverse} The constraints can be refined by following the same
steps as for calculating the inverse, but instead of replacing the (unknown)
pixel values in \cref{eq:constraints_pre} with $[0,1]$, we replace them with the
intervals calculated previously.
However, replacing $[x_l, x_u] \times [y_l, y_u]$ with the corner furthest away
from $(i,j)$ would be unsound.
To be sound, one needs to consider all 4 corners of every non-empty intersection
$[x_l, x_u] \times [y_l, y_u]$ and join all interval constraints.
Similarly, we use the previously calculated constraint for $p_{i,j}$ instead of
$[0,1]$ in \cref{eq:constraints_final}.
This procedure can be repeated to further increase precision. The final
result after applying the refinement $10$ times is shown in
\cref{fig:Inverse_Calc:refined} representing the lower (left) and upper (right)
interval bound for all pixels.

\subsection{Example} \label{sec:approach:example}
\begin{wrapfigure}{r}{.3\textwidth}
  \vspace{-3.5em}
  \begin{minipage}{\linewidth}
    \centering
    \captionsetup[subfigure]{justification=centering}
    \begin{minipage}{0.49\linewidth}
      \centering
      \includegraphics[width=0.9\linewidth]{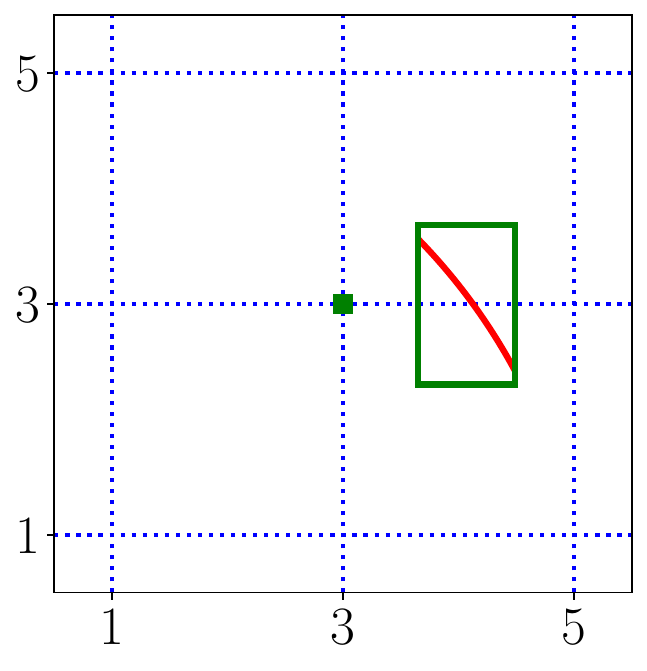}
      \subcaption{$c_{5,1}$} \label{fig:green_box}
    \end{minipage}
    \begin{minipage}{0.49\linewidth}
      \centering
      \includegraphics[width=0.9\linewidth]{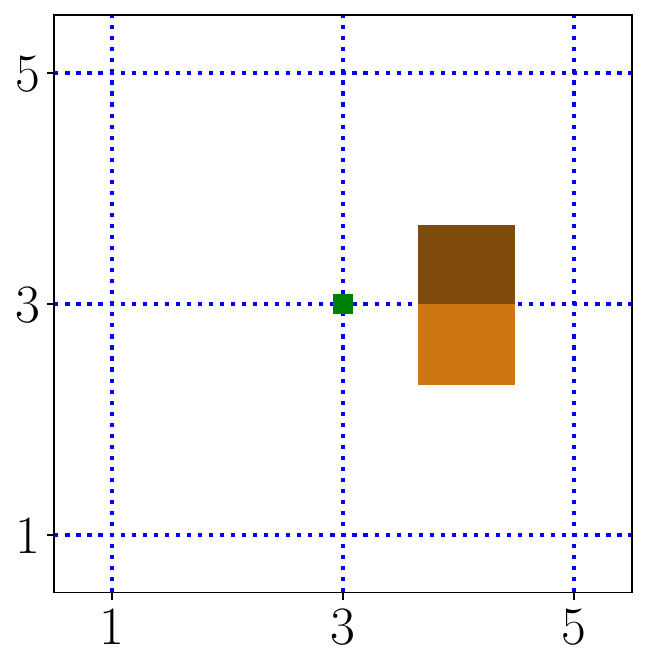}
      \subcaption{Intersections} \label{fig:intersections}
    \end{minipage}
  \end{minipage}
  \caption{To improve presentation, the red arc is $3\times$ longer.}
  \label{fig:inverse_calc}
  \vspace{-3em}
\end{wrapfigure}
We calculate the constraint for pixel $(3,3)$ of the original image
(\cref{fig:Inverse_Calc:original}), depicted as the green dot in
\cref{fig:inverse_calc} under the assumption $\gamma \in [23^\circ, 26^\circ]$.
We elaborate the constraints that pixel $(5, 1)$ of the rotated
image (\cref{fig:Inverse_Calc:rotated}) yields for pixel $(3,3)$ of the original image.

\paragraph{Step 1} We illustrate the calculation of the set $C$ for 
$c_{5,1} :=
  R_{[23^\circ, 26^\circ]}^{-1} \begin{psmallmatrix} 5 \\ 1 \end{psmallmatrix}
  =
  \begin{psmallmatrix}
    [4.06,4.21] \\
    [2.85,3.11]
  \end{psmallmatrix}
$.
The result is depicted as the green box in \cref{fig:green_box} enclosing the
red arc. The red arc shows the precise set of coordinates where the pixel value
$p'_{5,1}$ could have been interpolated from the original image $\vx$.

\paragraph{Step 2} The only non-empty intersections of $c_{5,1}$ with
interpolation regions (blue squares in \cref{fig:inverse_calc}), cornering $(3,3)$
are the $(3,1)$ and the $(3,3)$-interpolation regions, hence we omit
$q_{1,1}$ and $q_{1,3}$. The intersection with the
$(3,3)$-interpolation region yields $[x_l, x_u] = [4.06,4.21]$ and $[y_l, y_u] =
[3,3.11]$ (dark brown rectangle in \cref{fig:intersections}), hence
at the furthest corner $(x, y) = (4.21,3.11)$, we get
\begin{equation*}
  q_{3,3}
  =
  [0.73,2.48]
  =
  \left[
    p'_{5,1} - \left(
      \tfrac{5-x}{2} \tfrac{y-3}{2}
      + \tfrac{x-3}{2} \tfrac{5-y}{2}
      + \tfrac{x-3}{2}  \tfrac{y-3}{2}
    \right), p'_{5,1}
  \right]
  (\tfrac{5-x}{2} \tfrac{5-y}{2})^{-1}
  ,
\end{equation*}
and the intersection with the $(3,1)$-interpolation
region yields $[x_l,x_u] = [4.06,4.21]$ and $[y_l, y_u] = [2.85,3]$ (light brown
rectangle in \cref{fig:intersections}), hence at the furthest corner $(x,y) = (4.21,2.85)$, we get
\begin{equation*}
  q_{3,1}
  =
  [0.72,2.48]
  =
  \left[
    p'_{5,1} - \left(
      \tfrac{5-x}{2}\tfrac{3-y}{2}
      + \tfrac{x-3}{2} \tfrac{3-y}{2}
      + \tfrac{x-3}{2} \tfrac{y-1}{2}
      \right),
    p'_{5,1}
  \right]
  \left(
    \tfrac{5-x}{2} \tfrac{y-1}{2}
  \right)^{-1}
  .
\end{equation*}

\paragraph{Step 3} The join $q_{3,1} \sqcup q_{3,3}$ yields
$[0.72,2.48]$. After intersecting this with $[0,1]$ and the constraints from the
other $c_{i',j'} \in C$ (as in \cref{eq:constraints_final}), we are left with the final
result $p_{3,3} \in [0.73,1]$.

The final result of the inverse calculation for all pixels is shown in
\cref{fig:Inverse_Calc:inverse} representing the lower (left) and upper (right)
interval bounds for all pixels.
The iterative refinement is shown in \cref{fig:Inverse_Calc:refined}.

\section{Experimental Evaluation}\label{sec:evaluation} %
We now present our extensive evaluation of the different defenses discussed so far.

\subsection{Instantiation in Practice}\label{sec:inst-prat} %
In \cref{sec:approach} we showed how to certify robustness of $g_{E}$ to $T^{I}_{\gamma}$, obtained from \cref{eq:g} with $f := h_{E} \circ I$ and
$\psi_{\beta} := T_{\beta}$.  Since in practice $h_{E} \circ I$ and $T_{\beta}$ cannot be evaluated as $I$ and $T_{\beta}$ are not available independently,
in order to evaluate $g_E$ in practice, we need to re-write it as follows:
\begin{align*}
  g_E(\vx) &= \argmax_c
    \prob_{\beta~\sim~\mathcal{N}(0, \sigma^2 \mathbb{1})}
    \left(
      (h_{E} \circ I) \circ T_\beta(\vx) = c
  \right)\\
  &=
   \argmax_c
    \prob_{\beta~\sim~\mathcal{N}(0, \sigma^2 \mathbb{1})}
    \left(
      h_{E} \circ (I \circ T_\beta)(\vx) = c
    \right) =: g(\vx),
\end{align*}
which is an instantiation of \cref{eq:g} with $f := h_{E}$ and
$\psi_{\beta} := T^{I}_{\beta}$, both of which are available.

Further, as the probability in \cref{eq:g} cannot be computed
exactly, in practice we use the approximation introduced in
\citet{CohenRK19}: by taking $n$ samples around a given $\vx$ with
standard deviation $\sigma$, we can obtain $g(\vx)$ and the
corresponding robustness radius $r$ with confidence $1-\alpha$. Here, $n$
can be too small to make a statement with confidence $1-\alpha$, in
which case the classifier abstains.
Further, we let $\sigma_{\gamma}, \alpha_{\gamma}, n_{\gamma}, r_\gamma$ and
$\sigma_{\delta}, \alpha_{\delta}, n_{\delta}, r_\delta$
denote the parameters and radius required to use \cref{thm:bound} and \cref{thm:original} in practice, respectively.
Statistically sound certification as in \citet{CohenRK19} requires to fist take $n_{0}$ many samples fist and guessing the correct class on them. In our case we apply both $T^{I}_{\beta}$ and additive noise $\epsilon$ to these $n_{0}$ samples.

\subsection{Setup}\label{sec:setup} %
All experiments were performed on a machine with 2 GeForce RTX 2080 Tis
and an Intel(R) Core(TM) i9-9900K CPU.
As base classifiers $b$ we utilize neural networks in PyTorch
\cite{pytorch}, using \texttt{robustness} \cite{robustness} and \citet{Salman} for
training. Further, we implemented the interval analysis
(cf. \cref{sec:calc_error_bounds,subsec:inverse})
of the interpolation error and inverse computation in C++/CUDA.

We consider rotations $R_\gamma^{I}$ by $\gamma$
degrees and translations $\Delta_{\gamma}^{I}$ by
$\gamma \in \mathbb{R}^{2}$ with bilinear interpolation $I$. Here, we allow
the adversary to choose $\gamma \in \Gamma$.  For a scalar
$\Gamma_{\pm} \in \mathbb{R}^{\geq 0}$, we permit
$\Gamma := [-\Gamma_{\pm}, \Gamma_{\pm}]$ for rotations and
$\Gamma := [-\Gamma_{\pm}, \Gamma_{\pm}]^{2}$ for translations.
All estimates of $E$ include interpolation errors as well as 8-bit
representation (``rounding'') errors. When we estimate $\rho_{E}$ with confidence $\alpha_{E}$.

We evaluate on ImageNet \citep{ImageNet}, Restricted ImageNet
(RImageNet)\citep{TsiprasSETM19}, a subset of ImageNet with 10
classes, CIFAR-10 \citep{cifar}, and MNIST \citep{MNIST}. For the base classifier,
in \cref{sec:methodheuristic} we use standard models without any
additional training, while in the other sections we use models trained
with data augmentation (transformations, $\ell^2$-noise) using \cite{Salman}.

In \cref{sec:methodglobal,sec:methodindividual}, we apply a circular
or rectangular vignette for rotation and translation respectively, to
reduce error estimates in areas of the image where information is
lost. We also apply a Gaussian blur prior to classification to further
reduce the high-frequency components of the interpolation
error. \cref{sec:experiment-details} contains further details on
prepossessing, model training and parameters. Note that pre-processing
does not impact the theoretical guarantees as long as it is
consistently applied. We provide an ablation study regarding
vignetting and Gaussian blur in \cref{sec:furth-comp-ablat}.
Additional experiments, including other interpolation methods or audio
classification are provided in \cref{sec:addit-exper}, highlighting
the generality of our methods.
Throughout the section all individual certificates hold with overall confidence $1 - \alpha$ for $\alpha = 0.01$.

\subsection{\methodHeuristic}\label{sec:methodheuristic}

\begin{wraptable}{r}{0.5\textwidth}
  \vspace{-1.35em}
  \begin{small}
    \caption{Evaluation of \methodHeuristic.  We obtain Acc for $b$ on
      the test set and evaluate adv. Acc. on 3000 images obtained by
      the \texttt{worst-of-100} attack. $t$ denotes the
      average run time of $g$.} %
    \label{tab:heuristicResults}
    \begin{tabular}{@{}llccccc@{}} \toprule
      & & & Acc. &  \multicolumn{2}{c}{adv. Acc.}  & \\
      \cmidrule(lr){4-4} \cmidrule(lr){5-6}
      Dataset   & $T^{I}$ & $\Gamma_{\pm}$ & $b$ & $b$ & $g$ & t [s]\\
      \cmidrule(lr){1-7}
      MNIST     & $R^{I}$ & $30^{\circ}$ & 0.99 & 0.73 & 0.99 & 0.97\\
      CIFAR-10     & $R^{I}$ & $30^{\circ}$ & 0.91 & 0.26 & 0.85 & 0.95\\
      ImageNet  & $R^{I}$ & $30^{\circ}$ & 0.76 & 0.56 & 0.76 & 5.43\\
      \cmidrule(lr){1-7}
      MNIST     & $\Delta^{I}$ & $4$ & 0.99 & 0.03 & 0.53 & 0.86\\
      CIFAR-10     & $\Delta^{I}$ & $4$  & 0.91 & 0.44 & 0.79 & 0.95\\
      ImageNet  & $\Delta^{I}$ & $20$ & 0.76 & 0.65 & 0.75 & 6.70\\
      \bottomrule
    \end{tabular}
  \end{small}
  \vspace{-1em}
\end{wraptable}

We can quickly obtain a well-motivated but empirical defense by
instantiating \cref{thm:bound} with $\psi_\beta := T^{I}_{\beta}$ and
ignoring both the interpolation error \cref{eq:epsilon_error} and the
construction in \cref{sec:approach}.
\cref{tab:heuristicResults} shows results on an undefended classifier
$b$ and the \methodHeuristic smoothed version $g$. Here \emph{Acc.} is
obtained over the whole dataset. To evaluate \emph{adv. Acc.} we use
the \texttt{worst-of-k} proposed by \citet{EngstromMadryRotation},
which returns the $\gamma$ yielding the highest cross-entropy loss out
of $k$ randomly sampled $\gamma \sim \mathcal{U}(\Gamma)$. We apply
\texttt{worst-of-k} to 1000 images and produce 3 attacked images each,
resulting 3000 samples on which we then evaluate $b$ and $g$.
For $g$, the average inference time per image $t$ is generally fast,
where most time is spent on sampling transformations. The actual
inference, invoking $b$ on the samples, is not slowed down as all
samples fit into a single batch.
In this section we use
$n_{\gamma} = 1000, \sigma_\gamma = \Gamma_{\pm}$.
and $\alpha_{\gamma} = 0.01$.

We do not obtain certificates here as the assumptions of
\cref{thm:bound} are violated. However, we investigate in
\cref{sec:addit-exper} if the certification radius holds practically.

\begin{table}[ht]
  \centering
\begin{small}
  \caption{Evaluation of \methodGlobalD for $T^{I} := R^{I}$. We show the test set accuracy of $b$, certified accuracy of $g$ at different radii $r_{\gamma}$, along with the average run time $t$. $^{\#}$ denotes values obtained by sampling. Each certificate hold with overall confidence $0.99$.}
    \label{tab:rotResults} %

    \begin{tabular}{@{}lrrrrrrrrr@{}} \toprule
    & & & & \multicolumn{4}{c}{$g$ cert. acc at $r_{\gamma}$} \\
    \cmidrule(lr){5-8}
    Dataset & $E$ & $q_{E}$ & $b$ acc. & $0^{\circ}$ & $10^{\circ}$ & $20^{\circ}$ & $30^{\circ}$ & $t$ [s] & $n_{\gamma}$\\
      \midrule
    MNIST &   0.45  & 0.99 & 0.98 & 0.89 & 0.88 & 0.87 & 0.85 & 21.56 & 200\\
    CIFAR-10 &   0.55  & 0.99 & 0.56 & 0.31 & 0.28 & 0.25 & 0.19 & 89.75  & 50\\
    CIFAR-10 &   0.55  & 0.99 & 0.56 & 0.32 & 0.30& 0.28 & 0.25 & 351.47 & 200\\
    RImageNet & $1.20^{\#}$  & 0.97 & 0.78 & 0.74 & 0.72 & 0.68 & 0.61 & 100.73 & 50\\
    RImageNet & $1.35^{\#}$  & 0.99 & 0.78 & 0.64 & 0.62 & 0.56 & 0.50 & 100.13 & 50 \\
    ImageNet & $0.95^{\#}$  & 0.75 & 0.38 & 0.30 & 0.24 & 0.18 & 0.12 & 100.21 & 50\\
    ImageNet & $1.20^{\#}$  & 0.97 & 0.38 & 0.23 & 0.19 & 0.13 & 0.09 & 100.73 & 50\\
    ImageNet & $1.35^{\#}$  & 0.99 & 0.38 & 0.16 & 0.12 & 0.08 & 0.06 & 100.44 & 50\\
\bottomrule
\end{tabular}

\end{small}
\end{table}

\subsection{\methodGlobal\protect\footnote{\label{fn:eval}The results in \cref{sec:methodglobal,sec:methodindividual} differ from those in the version published at NeurIPS'20 due to an implementation bug we since fixed.
Further, we improved readability and provide additional results enable better comparison. A version of \cref{tab:rotResults} in the original layout can be found in \cref{sec:addit-methodglobal}.  }}
\label{sec:methodglobal}%
Here we evaluate \methodGlobalD and \methodGlobalx and compare with related approaches.

\paragraph{\methodGlobalD}
First we consider \methodGlobalD, where $E$ is obtained over the training set and expected to hold in distribution as discussed in \cref{sec:approach:distributional_bounds}.
This allows to run both prediction, where the robust accuracy shown here can be expected to hold in distribution, as well as certification (e.g. to show $g(\vx) = g(\vx')$) at inference time.

\cref{tab:rotResults}
shows our results for \methodGlobalD with
rotations.
We restrict the attacker model to
$\Gamma_{\pm} = 90^{\circ}$ for MNIST and $\Gamma_{\pm} = 30^{\circ}$ for other datasets.

To obtain $E$, we first sample the interpolation error in \cref{eq:maxEps} (using 1000 images). Subsequently, we choose $E$ slightly larger than this error.
With $E$ fixed, we test for $\rho_{E} = 0.001$ and expect $q_{E}$ to be close to 1
for all datasets. \cref{tab:rotResults} shows $q_{E}$ obtained with confidence
$1-\alpha_{E} = 0.999$ by using $1000$ samples for $\vx$ and $8000$ for $\beta$ (and correction
for possible test errors over $\beta$).
For small images, these bounds can be computed quickly.
However, for large images (ImageNet), the optimization over $\gamma$ for
many images is computationally expensive.
Thus, for ImageNet we replace the $\max$ in \cref{eq:qE} with the
maximum over 10 samples $\gamma \in \mathcal{U}(\Gamma)$ (indicated by $^{\#}$ in \cref{tab:rotResults}).
This formally restricts the certificate to only hold against random
attacks (such as \texttt{worst-of-10}).
However, if sufficient computational resources are available, the
$\max$ method can still be applied (we empirically find the method to
obtain similar values).
On (R)ImageNet (variable image size) we resize all images so that the
short side is 512 pixel prior to applying transformations.
As RImageNet is a subset of ImageNet, we use $E$ obtained on the later.

Now, we evaluate the accuracy of $b$ and $g$. For $b$ we use
the whole test set, while for $g$ we use 1000 samples.
In addition to the results in \cref{tab:rotResults}, at
$r_{\gamma} = 50$, the MNIST $g$ in this configuration still achieves $0.75$ certified accuracy.
Comparing the results on ImageNet and RImageNet shows that the limiting factor for our
method is the robustness of the base classifier, not the
size of the image.

We use $\sigma_{\gamma} = 30$ for all datasets and
$\sigma_{\delta} = 0.25, n_{\delta} = 10000, n_{0} = 10000$ for MNIST, %
$\sigma_{\delta} = 0.3, n_{\delta} = 15000, n_{0} = 10000$ for CIFAR-10 %
and $\sigma_{\delta} = 0.5, n_{\delta} = 2500, n_{0} = 200$ for (R)ImageNet in all but the $E=1.35$ setting where we use $\sigma_{\delta} = 0.55$.
We use $\alpha_{\gamma} = 0.005 - \alpha_{E}$ and $\alpha_{\delta} = \tfrac{0.005}{n_{\gamma}}$, such that the overall confidence for each certificate is $0.99$.
We expect these results to hold in distribution for at least $q_{E}$ percent of data points.

To showcase that \methodGlobalD can be applied as an online defense to obtain individual certificates $g(\vx) = g(\vx')$, we also evaluate on attacked images.
Using the same settings as above we can certify for 91 out of 100 MNIST images, adversarially rotated with $\Gamma_{\pm} = 30$, that they are classified the same as the original, while also being correct.

\begin{table}[ht]
  \centering
\begin{small}
  \caption{Evaluation of \methodGlobalx for $T^{I} := R^{I}$. We show the test set certified accuracy of $g$ at different radii $r_{\gamma}$, along with the average $E$ estimated, the average time $t_{E}$ to estimate $E$ and average time $t_{RS}$ to apply randomized smoothing.
    $^{\#}$ denotes values obtained by sampling. $^{*}$ we use a server with 128 threads on an AMD EPYC 7601 processor, on the same system as the other results these take 766 s. Each certificate hold with overall confidence $0.99$.}

    \label{tab:methodGlobalx} %
    \begin{tabular}{@{}lrrrrrrrrrrr@{}} \toprule
    & & & \multicolumn{5}{c}{$g$ cert. acc at $r_{\gamma}$} \\
    \cmidrule(lr){4-8}
    Dataset & $\Gamma_{\pm}$ & $\sigma_{\gamma}$ & $0^{\circ}$ & $10^{\circ}$ & $20^{\circ}$ & $30^{\circ}$ & $50^{\circ}$ & avg. $E$ & $t_{E}$ [s] & $t_{RS}$ [s] & $n_{\gamma}$\\
      \midrule
    MNIST & $50^{\circ}$ & 30 & 0.93 & 0.92 & 0.91 & 0.90 & 0.82 & 0.34 & 53.33 & 20.56 & 200\\
    CIFAR-10 & $30^{\circ}$ & 40 & 0.35 & 0.30 & 0.27 & 0.22 & - & 0.34 & 81.83 & 91.72 & 50\\
    CIFAR-10 & $10^{\circ}$  & 10 & 0.43 & 0.37 & - & - & - & 0.34 & 51.12 & 92.83 & 50 \\
    ImageNet & $30^{\circ}$ & 30  & 0.31 & 0.25 & 0.17 & 0.11 & - & 0.86$^{\#}$ & 73.58$^{*}$ & 100.47 & 50\\
    ImageNet & $30^{\circ}$ & 30  & 0.32 & 0.29 & 0.22 & 0.16 & - & 0.86$^{\#}$ & 73.58$^{*}$ & 396.50 & 200\\
\bottomrule
\end{tabular}

\end{small}
\end{table}

Finally, we evaluate translations on MNIST ($E = 0.65$, $\rho_{E}=0.99$, $\Gamma_{\pm} = 2$) and achieve  certified accuracy 64\% and 49\% at
$r_{\gamma}$ of 0 and $\sqrt{2}$, respectively. We use $\sigma_{\gamma} = 1.5, \sigma_{\delta} = 0.25, n_{\gamma} = 200$
and the other parameters as for rotation.

\paragraph{\methodGlobalx}
We now evaluate \methodGlobalx.
Here, we certify a classifier $g$ (for a fixed $b$, $\sigma_{\gamma}$, $\sigma_{\epsilon}$) on the test set. At inference time we just predict new samples and expect the obtained robustness certificates to hold in distribution. We show the certification results in \cref{tab:methodGlobalx}.

To this end, it is sufficient to obtain $E, \rho_{E}$ (as in \cref{eq:qE}) for each individual image $\vx$ rather than for the whole data distribution.
Naturally, these individual bounds are much lower than $E$ obtained over the data distribution, allowing for better accuracy.

On MNIST and CIFAR-10, for each image $\vx$ we use $100$ samples of $\beta$ (optimizing over $\gamma$) to guess $E$ as $1.1$ times the largest observed error.
Then, we use another $400$ samples of $\beta$ to test for $\rho_{E}$ with $\alpha_{E}=0.001$.
On ImageNet we use the same procedure but chose $E$ as the largest observed error over $240$ samples of $\beta$ plus $0.03$. As for \methodGlobalD, we use sampling to approximate the maximization.
When optimizing for $E$ we stop either when the highest bound for any $\beta$ is $0.3$ or after a timeout of $2$ minutes. Varying these parameters may allow for an even lower $E$ at the cost of more run time.

For translation ($\sigma_{\gamma} = 0.25, n_{\gamma} = 200, \Gamma_{\pm} = 2, \sigma_{\gamma} = 2.5$) we use the same setup but optimize until the maximal error is lower than $0.55$ or a timeout of 2 minutes is reached. We choose $E$ as the maximal error over the first 100 images plus $0.02$. With an average $E$ of $0.56$, we obtain a certified accuracy of $0.89$, $0.86$, $0.85$ and $0.82$ at radii $r_{\gamma}$ of $0$, $1$, $\sqrt{2}$, and $2$, respectively. The average analysis took $62.22$s and certification $19.33$s. We note that in general \methodGlobalD results are a lower bound for the results of \methodGlobalx. In theory, \methodGlobalD can perform better if $\rho_{E}$ is lower (e.g. when more samples are used). However, in practice this is offset by the tighter error bound. We see that the average $E$ is much lower than the upper bound used in \methodGlobalD, allowing better results.

Unless stated differently in \cref{tab:methodGlobalx}, we use the same parameters as for \methodGlobalD, with the exception of $\sigma_{\delta}$ where we use $0.15, 0.20, 0.18, 0.50, 0.40$ for the order as in  \cref{tab:methodGlobalx}.
As before, all certificates hold with confidence $0.99$ as $\alpha_{\gamma} = 0.005 - \alpha_{E}$ and $\alpha_{\delta} = \tfrac{0.005}{n_{\gamma}}$.

\paragraph{Comparison to other work} %
Related approaches, \citet{DeepG,LinyiLiSmoothing}, provide distribution certificates, \eg they certify images on the test set and the obtained certified accuracy can then be expected to hold for new, potentially perturbed images. However, it is not possible to certify novel inputs -- this is the same setting as with \methodGlobalx. \citet{DeepG} certifies model accuracy on the test set and thus provides a distributional bound. On MNIST they report $87.01\%$ of
certified accuracy for rotations with $\pm 30^{\circ}$ ($35s$ per image), which with further refinement (at cost of run time) can be increased to $97\%$, and for translations with $\pm 2$ pixels $76.30 \%$ ($263$s per image). On CIFAR-10 they certify rotation up to $10^{\circ}$ for $62.51 \%$, but unlike our work, the method does not scale to larger image sizes and models, such as ResNet-50 on ImageNet.  We provide further comparison with \citet{DeepG} in \cref{sec:furth-comp-ablat}. For a comparison with \cite{LinyiLiSmoothing,TSS}, we refer the reader to that work. \citet{PeiCYJ17} certify $\pm 2^{\circ}$ in 714 s per image on ImageNet. However, in contrast to us they focus on nearest-neighbor interpolation, which can be enumerated.

\subsection{\methodIndividual\protect\cref{fn:eval}} \label{sec:methodindividual}
Finally, we evaluate \methodIndividual, where we compute $E$ on the given input.
The bound computed by interval analysis is always sound, but may be quite large due to the loss of precision inherent in interval analysis.
We show results for MNIST and discuss challenges on larger datasets in \cref{sec:inverse-rich-images}.
To this end, we attack images as in \cref{sec:methodheuristic}, and
subsequently apply \methodIndividual. We use the \texttt{worst-of-100}
attack on a base classifier $b$ to obtain a set of attacked images.
To these images we then apply \methodIndividual.
For rotations
($\Gamma_{\pm}=10, \sigma_{\gamma}=30, \sigma_{\delta}=0.3, n_{\gamma} = 200, n_{\delta}=10000, n_{0} = 10000$, 3 attacks per image, 1000 images) we fix $E=0.45$ and use 500 samples of $\beta$
to obtain the correct $\rho_{E}$ (\cref{eq:indiv_err_bound}) with $\alpha_{E} = 0.001$. $g$
was correct on 82\% of attacked  images. For 81\% we could certify
that the attacked image that classifies the same as the original.  The
analysis of $E$ took on average 0.26 s and the randomized smoothing
25.03 s.
For translation we use the same setup
($\Gamma_{\pm}=1, \sigma_{\gamma}=1.5, \sigma_\delta = 0.3, n_\gamma =
200$, 3 attacks per image, 100 images) also starting with $E=0.45$. $g$ classified 75\%  of attacked images correctly and could certify
$r_{\gamma} \geq 1$ and thereby $g(\vx) = g(\vx')$ on
$59\%$
while on average taking 14.14 s for analysis and 19.89 s
for smoothing per image.  The reason for the higher run time is that
compared to rotation fewer possible inverses can be discarded.
We use 10 refinement steps for both rotations and translations.

\subsection{Limitations \& Generalization} %

While we showcased translation and rotation, our approach is not
limited to these transformations or to specific interpolation methods. \methodHeuristic
and \methodGlobal can be directly adapted to other transformations,
interpolation schemes or domains such as audio (see \cref{sec:addit-exper}).
\methodIndividual can also be adapted but requires additional care.
Generally, \cref{thm:bound} can be applied to all parameterized data
transformations that are additive in the parameter space. If this
holds up to a small error, as discussed here, \methodGlobal and
\methodIndividual can be applied.  While many data transformations,
\eg image scaling are additive in their parameter space, their
compositions are often not (\eg rotation and translation).
As we are most limited by the $\ell^{p}$-robustness of $b$, any gains
in $\ell^{p}$ certification will directly improve our method. Further,
\methodIndividual can incur a large loss of precision in the inverse
computation. Improving this directly increases the applicability of
the method.

\section{Conclusion}
We presented the first generalization of randomized smoothing to image transformations, a challenging task as image transformations do not compose. Based on this generalization, we presented several certified defenses
allowing for both distributional and individual guarantees (relying on statistical error bounds or on efficient inverse computation). Our exploration highlights interesting trade-offs between certification guarantees and tightness of the resulting bounds. Finally, our extensive evaluation demonstrates the methods can handle realistic datasets and models.

\section{Broader Impact}

In general, methods from artificial intelligence can be applied in beneficial and
malicious ways. While this poses a threat in itself, verification techniques
provide formal guarantees for the robustness of the model, independently of the
intended use case. Certification techniques could therefore distinguish a
potentially unstable model from a stable one in safety critical settings, \eg
autonomous driving. However, especially for regulators, it is of utter
importance to understand the certified properties of different
certification methods precisely, as to avoid legal model deployment in safety
critical applications based on misconceptions.

\begin{ack}
  We thank the authors of \cite{LinyiLiSmoothing}, in particular Maurice Weber and Linyi Li, for insightful discussion and pointing out an implementation bug. Further, we thank all reviewers for their helpful comments and feedback.

  We do not have any additional funding or compensation to disclose.
\end{ack}

\message{^^JLASTBODYPAGE \thepage^^J}

\clearpage
\bibliography{references}
\bibliographystyle{unsrtnat}

\message{^^JLASTREFERENCESPAGE \thepage^^J}

\ifbool{includeappendix}{%
	\clearpage
	\appendix

  {
    \centering
    {\LARGE\bf Supplementary Material for\\ Certified Defense to Image Transformations via Randomized Smoothing\par}
  }
	\section{Proof of \cref{thm:bound}} \label{proof}

We now proceed to proof \cref{thm:bound}. We achieve this by first
proofing an auxiliary Theorem and Lemma, and then instantiating as a
special case \cref{thm:bound} of these slightly more general results.

\begin{theorem} \label{thm:aux}
  Let $\vx \in \mathbb{R}^n$, $f: \mathbb{R}^m \to \mathcal{Y}$ be a
  classifier, $\psi_\beta: \mathbb{R}^n \to \mathbb{R}^m$ a composable
  transformation for $\beta\,\sim\,\mathcal{N}(0, \Sigma)$ with a
  symmetric, positive-definite covariance matrix
  $\Sigma \in \mathbb{R}^{m \times m}$. If
  \begin{align*}
    &\prob_\beta(f \circ \psi_\beta(x) = c_{A})
    =
    p_A
    \geq
      \underline{p_{A}}
      \geq
    \overline{p_{B}}
    \geq
    p_B
    =
    \max_{c_{B} \neq c_{A}} \prob_\beta(f \circ \psi_\beta(x) = c_{B}),
  \end{align*}
  then $g \circ \psi_\gamma(\vx) = c_A$ for all $\gamma$ satisfying
  \begin{equation*}
    \sqrt{\gamma^{T} \Sigma^{-1} \gamma} < \tfrac{1}{2}(\Phi^{-1}(\underline{p_{A}}) -
    \Phi^{-1}(\overline{p_{B}})) =: r_\gamma.
  \end{equation*}
  \begin{proof}
    
The assumption is 
\begin{align*}
    &\prob \left( \left( f \circ \psi_{\beta} \right) (\vx) = c_{A} \right)
    = p_{A} \geq \underline{p_{A}} 
    \geq 
    \overline{p_{B}} 
    \geq p_{B} = 
    \prob \left( \left( f \circ \psi_{\beta} \right) (\vx) = c_{B} \right).  
\end{align*}

By the definition of $g$ we need to show that
\begin{equation*}
  \prob \left( \left( f \circ \psi_{\beta + \gamma} \right) (\vx) = c_{A} \right)
  \geq
  \prob \left( \left( f \circ \psi_{\beta + \gamma} \right) (\vx) = c_{B} \right).
\end{equation*}
We define the set
$A := \{ \vz \mid \gamma^T \Sigma^{-1} \vz \leq \sqrt{\gamma^T \Sigma^{-1} \gamma} \Phi(\underline{p_{A}}) \}$.
We claim that for $\beta \sim \mathcal{N}(0, \Sigma)$, we have
\begin{align}
  \prob(\beta \in A) = \underline{p_{A}}
  \label{eq:app:1}
  \\
  \prob(f \circ \psi_{\beta + \gamma}(x) = c_{A}) &\geq \prob(\beta + \gamma \in A).
  \label{eq:app:2}
\end{align}
First, we show that \cref{eq:app:1} holds. 
\begin{align*}
  \prob(\beta \in A)
  &= 
  \prob(
      \gamma^T \Sigma^{-1} \beta 
      \leq 
      \sqrt{\gamma^T \Sigma^{-1} \gamma} \Phi(\underline{p_{A}})
  )
  \\
  &= 
  \prob(
      \gamma^T \Sigma^{-1} \mathcal{N}(0, \Sigma) 
      \leq 
      \sqrt{\gamma^T \Sigma^{-1} \gamma} \Phi(\underline{p_{A}})
  )
  \\
  &= 
  \prob(
      \gamma^T \sqrt{\Sigma^{-1}} \mathcal{N}(0, \mathbb{1}) 
      \leq 
      \sqrt{\gamma^T \Sigma^{-1} \gamma} \Phi(\underline{p_{A}})
  )
  \\
  &= 
  \prob(
      \mathcal{N}(0, \gamma^T \Sigma^{-1} \gamma) 
      \leq 
      \sqrt{\gamma^T \Sigma^{-1} \gamma} \Phi(\underline{p_{A}})
  )
  \\
  &= 
  \prob(
      \sqrt{\gamma^T \Sigma^{-1} \gamma} \mathcal{N}(0, 1) 
      \leq 
      \sqrt{\gamma^T \Sigma^{-1} \gamma} \Phi(\underline{p_{A}})
  )
  \\
  &= 
  \prob(
      \mathcal{N}(0, 1) 
      \leq 
      \Phi(\underline{p_{A}})
  )
  \\
  &= \Phi(\Phi^{-1}(\underline{p_{A}})) 
  \\
  &= \underline{p_{A}}
\end{align*}

Thus \cref{eq:app:1} holds. Next we show that \cref{eq:app:2} holds.
For a random variable $v\,\sim\,\mathcal{N}(\mu_{v}, \Sigma_{v})$ we
write $p_{v}(z)$ for the evaluation of the Gaussian cdf at point
$z$.

\begin{align*}
  &\prob(f \circ \psi_{\beta + \gamma}(x) = c_A) - \prob(\beta + \gamma \in A)
  \\
  &= \int_{\mathbb{R}^d} [f \circ \psi_\vz = c_A]\, p_{\beta + \gamma}(z) dz - \int_A p_{\beta+\gamma}(z) dz 
  \\
  &= \int_{\mathbb{R}^d \setminus A} [f \circ \psi_\vz(x)=c_{A}]\, p_{\beta + \gamma}(z) dz + \int_A [f \circ \psi_\vz(x)=c_{A}]\, p_{\beta + \gamma}(z) dz - \int_A p_{\beta+\gamma}(z) dz\\
  &= \int_{\mathbb{R}^d \setminus A} [f \circ \psi_\vz(x)=c_{A}]\, p_{\beta + \gamma}(z) dz + \int_A [f \circ \psi_\vz(x)=c_{A}]\, p_{\beta + \gamma}(z) dz\\
  &\qquad- \left( \int_A [f \circ \psi_\vz(x)=c_{A}]\,  p_{\beta+\gamma}(z) dz + \int_A [f \circ \psi_\vz(x)\neq c_A]\,  p_{\beta+\gamma}(z) dz  \right)\\
  &= \int_{\mathbb{R}^d \setminus A} [f \circ \psi_\vz(x)=c_{A}]\, p_{\beta + \gamma}(z) dz - \int_A [f \circ \psi_\vz(x)\neq c_A]\,  p_{\beta+\gamma}(z) dz\\
  &\stackrel{\cref{lem:app:1}}{\geq}
  t \left(
  \int_{\mathbb{R}^d \setminus A} [f \circ \psi_\vz(x)=c_{A}]\, p_\beta(z) dz
  - \int_A [f \circ \psi_\vz(x) \neq c_A]\, p_\beta(z) dz
  \right)
  \\
  &=
  t \left(
  \int_{\mathbb{R}^d} [f \circ \psi_\vz(x)=c_{A}]\, p_\beta(z) dz
  - \int_A p_\beta(z) dz
  \right)
  \\
  &\overset{\cref{eq:app:1}}{\geq} 0.
\end{align*}
Thus also \cref{eq:app:2} holds.

Next, we claim that for $B := \{z \mid \gamma^T \Sigma^{-1} \vz  \geq
\sqrt{\gamma^T \Sigma^{-1} \gamma} \Phi^{-1}(1 - \overline{p_{B}})\}$ holds that 
\begin{align}
  \prob (f \circ \psi_\beta(x) = c_{B}) &\leq \prob (\beta \in B) 
  \label{eq:app:3}
  \\
  \prob (f \circ \psi_{\beta + \gamma}(x) = c_{B}) &\leq \prob (\beta + \gamma \in B)
  \label{eq:app:4}
\end{align}
The proofs for \cref{eq:app:3} and \cref{eq:app:4} are analogous to the proofs
for \cref{eq:app:1} and \cref{eq:app:2}.

Now we derive the conditions that lead to $\prob(\beta + \gamma \in A) > \prob(\beta + \gamma \in B)$:
\begin{align*}
  \prob(\beta + \gamma \in A) 
  &= \prob \left(\gamma^T \Sigma^{-1}(\beta + \gamma) \leq 
  \sqrt{\gamma^T \Sigma^{-1} \gamma} \Phi^{-1}(\underline{p_{A}}) \right)
  \\
  &= \prob \left(\gamma^T \Sigma^{-1}(\Sigma^{\frac{1}{2}} \mathcal{N}(0, \mathbb{1}) + \gamma) \leq 
  \sqrt{\gamma^T \Sigma^{-1} \gamma} \Phi^{-1}(\underline{p_{A}}) \right)
  \\
  &= \prob \left(
    \gamma^T \sqrt{\Sigma^{-1}} \mathcal{N}(0, \mathbb{1}) + \gamma^T \Sigma^{-1} \gamma 
    \leq 
    \sqrt{\gamma^T \Sigma^{-1} \gamma} \Phi^{-1}(\underline{p_{A}}) \right)
  \\
  &= \prob \left(
    \sqrt{\gamma^T \Sigma^{-1} \gamma} \mathcal{N}(0, \mathbb{1}) + \gamma^T \Sigma^{-1} \gamma 
    \leq 
    \sqrt{\gamma^T \Sigma^{-1} \gamma} \Phi^{-1}(\underline{p_{A}}) \right)
  \\
  &= \prob \left(
    \mathcal{N}(0, \mathbb{1}) + \sqrt{\gamma^T \Sigma^{-1} \gamma}
    \leq 
    \Phi^{-1}(\underline{p_{A}}) \right)
  \\
  &= \prob \left(
    \mathcal{N}(0, \mathbb{1})
    \leq 
    \Phi^{-1}(\underline{p_{A}}) - \sqrt{\gamma^T \Sigma^{-1} \gamma} \right)
  \\
  &= \Phi(
    \Phi^{-1}(\underline{p_{A}}) - \sqrt{\gamma^T \Sigma^{-1} \gamma})
\end{align*}

Similarly, we have
\begin{align*}
  \prob(\beta + \gamma \in B) 
  &=
  \prob \left(
    \mathcal{N}(0, \mathbb{1})
    \geq 
    \Phi^{-1}(1 - \overline{p_{B}}) - \sqrt{\gamma^T \Sigma^{-1} \gamma} \right) 
  \\
  &= \Phi( \sqrt{\gamma^T \Sigma^{-1} \gamma} -
    \Phi^{-1}(1 - \overline{p_{B}}))
\end{align*}

Thus, we get
\begin{alignat*}{2}
  & \qquad \qquad \qquad \prob(\beta + \gamma \in A) &&> \prob(\beta + \gamma \in B)
  \\
  \Leftrightarrow & \qquad
  \Phi(\Phi^{-1}(\underline{p_{A}}) - \sqrt{\gamma^T \Sigma^{-1} \gamma})
  &&> 
  \Phi(\sqrt{\gamma^T \Sigma^{-1} \gamma}-\Phi^{-1}(1 - \overline{p_{B}}))
  \\
  \Leftrightarrow & \qquad
  \Phi^{-1}(\underline{p_{A}}) - \sqrt{\gamma^T \Sigma^{-1} \gamma}
  &&> 
  \sqrt{\gamma^T \Sigma^{-1} \gamma}-\Phi^{-1}(1 - \overline{p_{B}})
  \\
  \Leftrightarrow & \qquad
  \Phi^{-1}(\underline{p_{A}})+\Phi^{-1}(1 - \overline{p_{B}})
  &&> 
  2 \sqrt{\gamma^T \Sigma^{-1} \gamma}
  \\
  \Leftrightarrow & \qquad
  \tfrac{1}{2}(\Phi^{-1}(\underline{p_{A}})-\Phi^{-1}(\overline{p_{B}}))
  &&> 
  \sqrt{\gamma^T \Sigma^{-1} \gamma}.
\end{alignat*}

\end{proof}

\end{theorem}

Next ,we show the lemma used in the proof.

\begin{lemma} \label{lem:app:1}
  There exists $t > 0$ such that $p_{\beta + \gamma}(z) \leq p_{\beta}(z) \cdot t$ for all $z \in A$. And further $p_{\beta + \gamma}(z) > p_{\beta}(z) \cdot t$ for all $z \in \mathbb{R}^{d} \setminus A$.
\end{lemma}

\begin{proof}
  \begin{align*}
    \frac{p_{\beta + \gamma}(z)}{p_\beta(z)} 
    &= \exp \left( 
      - \tfrac{1}{2} (\vz - \gamma)^T \Sigma^{-1} (\vz - \gamma) + \tfrac{1}{2} \vz^T \Sigma^{-1} \vz 
    \right)
    \\
    &= \exp \left( 
    - \tfrac{1}{2} \vz^T \Sigma^{-1} \vz
    + \vz^T \Sigma^{-1} \gamma
    - \tfrac{1}{2} \gamma^T \Sigma^{-1} \gamma
    + \tfrac{1}{2} \vz^T \Sigma^{-1} \vz 
    \right)
    \\
    &= \exp \left( 
      \vz^T \Sigma^{-1} \gamma
      - \tfrac{1}{2} \gamma^T \Sigma^{-1} \gamma
      \right)
  \end{align*}
  
  What is the lowest $t$ if it exists such that $\frac{p_{\beta + \gamma}(z)}{p_\beta(z)} \leq t$?
  \begin{alignat*}{2}
    & \qquad \qquad \frac{p_{\beta + \gamma}(z)}{p_\beta(z)} 
    &&\leq t
    \\
    \Leftrightarrow \qquad
    & \exp \left(\vz^T \Sigma^{-1} \gamma - \tfrac{1}{2} \gamma^T \Sigma^{-1}
      \gamma \right) 
    &&\leq t
    \\
    \Leftrightarrow \qquad 
    & \quad \vz^T \Sigma^{-1} \gamma - \tfrac{1}{2} \gamma^T \Sigma^{-1} \gamma  
    &&\leq \log t
    \\
    \Leftrightarrow \qquad
    & \quad \quad \quad \vz^T \Sigma^{-1} \gamma 
    &&\leq \log t + \tfrac{1}{2} \gamma^T \Sigma^{-1} \gamma
    \\
  \end{alignat*}
  Because $z \in A$, we know that 
  \begin{equation*}
    \vz^T \Sigma^{-1} \gamma \leq \sqrt{\gamma^T \Sigma^{-1} \gamma} \Phi^{-1}(\underline{p_{A}}).
  \end{equation*}
  Does there exist a $t$ such that both upper bound coincide? Yes, namely
  \begin{equation*}
    t = \exp \left( 
      \sqrt{\gamma^T \Sigma^{-1} \gamma} \Phi^{-1}(\underline{p_{A}}) 
      - \tfrac{1}{2} \gamma^T \Sigma^{-1} \gamma  
    \right).
  \end{equation*}
  The case $p_{\beta + \gamma}(z) > p_{\beta}(z) \cdot t$ is analogous.
\end{proof}

\begin{lemma} \label{lemma:bound_prob}
  If we evaluate on a proxy classifier $f'$ instead of $f$, behaving with probability $(1 -
  \rho)$ the same as $f$ and with probability $\rho$ differently than $f$ and if
\begin{equation*}
  \prob_{\beta, f'}(f' \circ \psi_\beta(x) = c_{A})
  \geq
  \underline{p'_{A}}
  \geq \overline{p'_{B}}
  \geq
  \max_{c_{B} \neq c_{A}} \prob_{\beta, f'}(f' \circ \psi_\beta(x) = c_{B}),
\end{equation*}
then $g \circ \psi_\gamma(\vx) = c_A$ for all $\gamma$ satisfying
\begin{equation*}
  \|\gamma\|_2 < \frac{\sigma}{2}(\Phi^{-1}(\underline{p'_{A}} - \rho) -
\Phi^{-1}(\overline{p_{B}} + \rho)).
\end{equation*}

\begin{proof} 
  By applying the union bound we can relate the output probability $p$
  of $f$ for a class $c$ with the output probability of $f'$ and $p'$:
  \begin{align*}
    p' &:= \prob_{\beta, f'}(f' \circ \psi_\beta(x) = c)
    \\
    &= \prob_{\beta, f'}\left( (f \circ \psi_\beta(x) = c) \lor (f' \text{ error}) \right)
    \\
    &\leq \prob_\beta(f \circ \psi_\beta(x) = c) + \prob_{f'} (f' \text{ error})
    \\
    &= p + \rho
  \end{align*}
  Thus we can obtain new bounds
  $\underline{p_{A}} \geq \underline{p'_{A}} - \rho$ and
  $\overline{p_{B}} \leq \overline{p'_{B}} + \rho$ from
  $\underline{p'_{A}}$ and $\overline{p'_{B}}$ measured on
  $f'$. Plugging these bounds in \cref{thm:bound} yields the
  result.
\end{proof}

\end{lemma}

We now show \cref{thm:bound} (restarted below): Setting
$\Sigma = \sigma^{2} \mathbb{1} $ in \cref{thm:aux} directly recovers
\cref{thm:bound} up to the last sentence, which in turn is a direct
consequence of \cref{lemma:bound_prob}.

\begin{theorem*}[\cref{thm:bound} restated]
  Let $\vx \in \mathbb{R}^m$, $f: \mathbb{R}^m \to \mathcal{Y}$ be a classifier
  and $\psi_\beta: \mathbb{R}^m \to \mathbb{R}^m$ be a composable transformation
  as above. If
  \begin{equation*}
    \prob_\beta(f \circ \psi_\beta(\vx) = c_{A})
    \geq
    \underline{p_{A}}
    \geq
    \overline{p_{B}}
    \geq
    \max_{c_{B} \neq c_{A}} \prob_\beta(f \circ \psi_\beta(\vx) = c_{B}),
  \end{equation*}
  then $g \circ \psi_\gamma(\vx) = c_A$ for all $\gamma$ satisfying
  $ \|\gamma\|_2 \leq \tfrac{\sigma}{2}(\Phi^{-1}(\underline{p_{A}}) -
  \Phi^{-1}(\overline{p_{B}})) =: r_\gamma.  $ Further, if $g$ is
  evaluated on a proxy classifier $f'$ that behaves like $f$ with
  probability $1-\rho$ and else returns an arbitrary answer, then
  $r_\gamma := \tfrac{\sigma}{2}(\Phi^{-1}(\underline{p_{A}}-\rho) -
  \Phi^{-1}(\overline{p_{B}}+\rho))$.
\end{theorem*}

	\section{Inverse and Refinement}\label{sec:refimenet} %

\subsection{Details for Step 2}

In this section, we elaborate on the details of Step 2 in \cref{subsec:inverse}.
We consider the intersection of $c_{i',j'}$ with the $(i,j)$-interpolation
region, $[x_l,x_u] \times [y_l,y_u] := c_{i',j'} \cap [i,i+2] \times [j,j+2]$.
This yields, 
\begin{align*}
    p'_{i',j'} 
    \in 
    I([x_l,x_u],[y_l,y_u]) 
    = \;
    & p_{i,j}\tfrac{2+i-[x_l,x_u]}{2}\tfrac{2+j-[y_l,y_u]}{2}
    + p_{i,j+2}\tfrac{2+i-[x_l,x_u]}{2}\tfrac{[y_l,y_u]-j}{2}
    \\
    &+ p_{i+2,j}\tfrac{[x_l,x_u]-i}{2}\tfrac{2+j-[y_l,y_u]}{2}
    + p_{i+2,j+2}\tfrac{[x_l,x_u]-i}{2}\tfrac{[y_l,y_u]-j}{2}.
\end{align*}
Next, we solve for the pixel value $p_{i,j}$ to get the constraint $q_{i,j}$:
\begin{align*}
    q_{i,j} =
        \Big(
            p'_{i',j'} &- 
            p_{i,j+2}\tfrac{2+i-[x_l,x_u]}{2}\tfrac{[y_l,y_u]-j}{2}
            - p_{i+2,j}\tfrac{[x_l,x_u]-i}{2}\tfrac{2+j-[y_l,y_u]}{2}
            \\
            &- p_{i+2,j+2}\tfrac{[x_l,x_u]-i}{2}\tfrac{[y_l,y_u]-j}{2}
        \Big) \Big( 
            \tfrac{2+i-[x_l,x_u]}{2}\tfrac{2+j-[y_l,y_u]}{2} 
        \Big)^{-1}
\end{align*}
Because we don't have any constraints for the pixel values $p_{i+2,j},
p_{i,j+2}$ and $p_{i+2,j+2}$, we replace their values by the $[0,1]$
constraint and obtain: 
\begin{align*}
    q_{i,j} =
        \Big(
            p'_{i',j'} &- 
            \Big(\tfrac{2+i-[x_l,x_u]}{2}\tfrac{[y_l,y_u]-j}{2}
            - \tfrac{[x_l,x_u]-i}{2}\tfrac{2+j-[y_l,y_u]}{2}
            \\
            &- \tfrac{[x_l,x_u]-i}{2}\tfrac{[y_l,y_u]-j}{2} \Big) [0,1]
        \Big) \Big( 
            \tfrac{2+i-[x_l,x_u]}{2}\tfrac{2+j-[y_l,y_u]}{2} 
        \Big)^{-1}
\end{align*}
Instead of using standard interval analysis to compute the constraints for
$p_{i,j}$, we use the following more efficient transformer: We
replace $[x_l,x_u]$ and $[y_l,y_u]$ with the coordinate
$(x,y) \in [x_l,x_u] \times [y_l,y_u]$ furthest away from
$(i,j)$, which is in our case $(x_u, y_u)$ to obtain
\begin{align*}
    q_{i,j}
    &= 
        \Big(
            p'_{i',j'} - 
            \Big(\tfrac{2+i-x_u}{2}\tfrac{y_u-j}{2}
            + \tfrac{x_u-i}{2}\tfrac{2+j-y_u}{2}
            + \tfrac{x_u-i}{2}\tfrac{y_u-j}{2} \Big) [0,1]
        \Big) \Big( 
            \tfrac{2+i-x_u}{2}\tfrac{2+j-y_u}{2} 
        \Big)^{-1}
    \\
    &=
    \left[
        p'_{i',j'} - \left(\tfrac{2+i-x_u}{2}\tfrac{y_u-j}{2}
        + \tfrac{x_u-i}{2}\tfrac{2+j-y_u}{2}
        + \tfrac{x_u-i}{2}\tfrac{y_u-j}{2}\right) 
        , 
        p'_{i',j'}
    \right]
    \left(\tfrac{2+i-x_u}{2}\tfrac{2+j-y_u}{2}\right)^{-1}
    .
\end{align*}

\subsection{Algorithm}
Here, we present the algorithm used to compute the inverse of a transformation.
For the construction of the set $C$, we iterate only over the index set $P$. The
set $P$ is constructed do include all points in $G$ that could yield non empty
intersections $c_{i',j'}$, thus this is just to speed up the evaluation and
equivalent otherwise to the algorithm described in the main part.
\begin{algorithm}[H]
    \SetAlgoLined
    \LinesNumbered
    \DontPrintSemicolon
    \setstretch{1.35}
  \KwData{Image $\vx' \in \mathbb{R}^{m \times m}$, transform $T$, parameter range $B$, coordinates $i,j$}
  \KwResult{Range for the pixel value $p_{i,j}$.}

    $N \gets \begin{psmallmatrix}[i-2, i+2] \\ [j-2, j+2] \end{psmallmatrix}$
  
    $\begin{psmallmatrix} [i'_l, i'_u] \\ [j'_l, j'_u] \end{psmallmatrix} 
      \gets 
      T_B(N)  $

    $P \gets \left\{
      \begin{psmallmatrix} i' \\ j' \end{psmallmatrix}
      \Big|
      \begin{smallmatrix} 
        i' \in \text{range}(\lfloor i'_l \rfloor, \dots, \lceil i'_u \rceil,2) \\
        j' \in \text{range}(\lfloor j'_l \rfloor, \dots, \lceil j'_u \rceil,2)
      \end{smallmatrix}
    \right\}$ \;
  
    $C \gets \left\{ c_{i',j'} 
      := 
      T_B^{-1} \begin{psmallmatrix} i' \\ j' \end{psmallmatrix}
      \cap 
      N
      \Big| \, c_{i',j'} \neq \emptyset, 
      (i',j') \in P
      \right\}$

    $p_{i,j} \gets [0,1] \cap 
        \left( 
            \bigcap\limits_{c_{i',j'} \in C} 
                q_{i-2,j-2}(c_{i', j'})
                \cup q_{i,j-2}(c_{i', j'})
                \cup q_{i-2,j}(c_{i', j'})
                \cup q_{i,j}(c_{i', j'})
        \right)
    $
  
   \caption{Procedure to calculate the range for the pixel values of the inverse image.} \label{algorithm:optimized}
  \end{algorithm}

\subsection{Experimental Evaluation}\label{sec:exper-eval} %

To investigate the impact of refinement on the downstream error
estimate we used $20$ MNIST images, rotated each with $3$ random
angles and then proceeded to calculate the inverse. In the calculation,
we considered the range $\Gamma_{\pm} = 10$. We see that a low number of
refinements have a large impact on the error but the returns become
quickly diminishing. The impact on the run time of a single additional
refinement step is negligible.

\begin{figure*}[h]
  \centering
    \begin{subfigure}[t]{.45\textwidth}
        \centering
        \includegraphics[width=0.7\textwidth]{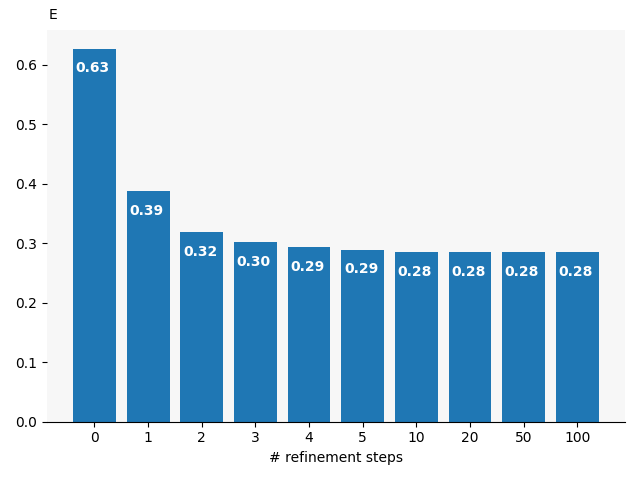}
    \end{subfigure}  
    \begin{subfigure}[t]{.45\textwidth}
      \centering
      \includegraphics[width=0.7\textwidth]{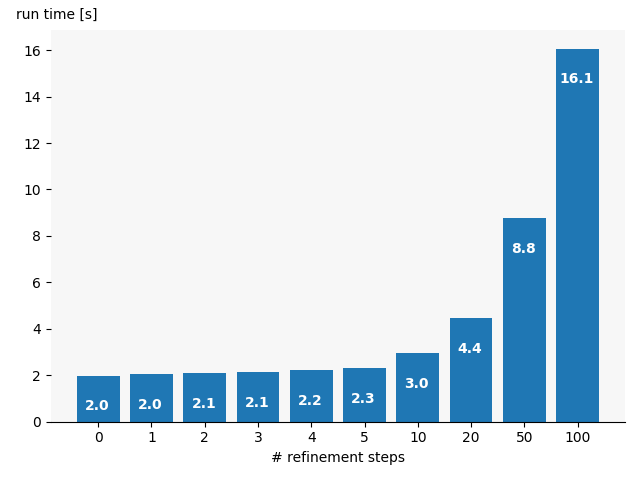}
    \end{subfigure}    
    \caption{Interpolation and rounding error $E$ as well as run time
      for different numbers of refinement steps.}
  \label{fig:refinement}
\end{figure*}

  \section{Inverse for Rich Images}\label{sec:inverse-rich-images} %
\methodIndividual performs poorly on large images, such as those from
ImageNet as the inverse computation outlined in \cref{subsec:inverse}
produces a too large over-approximation of $\vx$ leading to $E$
estimates of around $40$, while manageable value would be $\leq 2$.

\cref{fig:InverseIN} shows the computed inverse for such images. We
observe a pattern of artifacts in the inverse, where the pixel value
can not be narrowed down sufficiently resulting in the large estimate
of $E$.  The result of the refined inverse is perfectly recognizable
to a human observer (or a neural network), highlighting the promise of
the algorithm for future applications.

\begin{figure*}%
    \centering
    \begin{subfigure}[t]{.15\textwidth}
        \centering
        \includegraphics[width=0.7\textwidth]{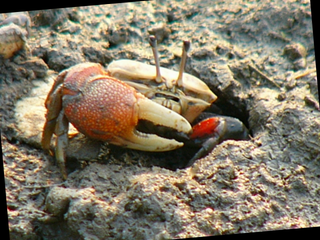}
    \end{subfigure}
    \begin{subfigure}[t]{.30\textwidth}
        \centering
        \includegraphics[width=0.35\textwidth]{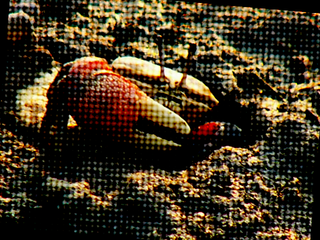}
        \includegraphics[width=0.35\textwidth]{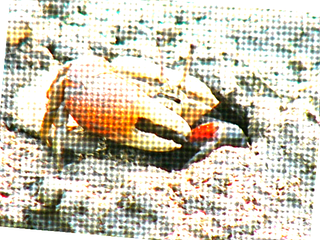}
    \end{subfigure}
    \begin{subfigure}[t]{.30\textwidth}
        \centering
        \includegraphics[width=0.35\textwidth]{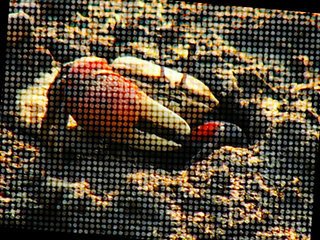}
        \includegraphics[width=0.35\textwidth]{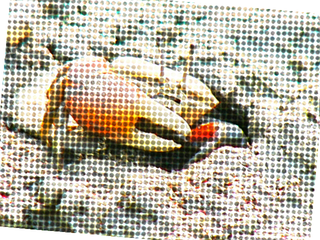}
    \end{subfigure}
    \begin{subfigure}[t]{.15\textwidth}
        \centering
        \includegraphics[width=0.7\textwidth]{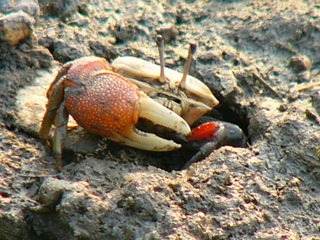}
      \end{subfigure}\\
    \begin{subfigure}[t]{.15\textwidth}
        \centering
        \includegraphics[width=0.7\textwidth]{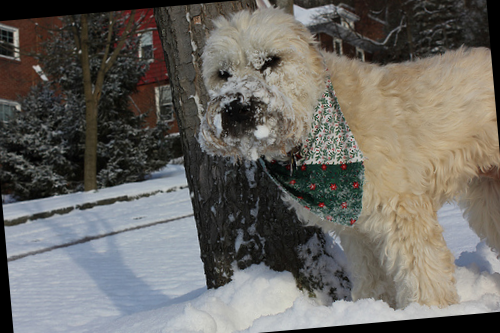}
    \end{subfigure}
    \begin{subfigure}[t]{.30\textwidth}
        \centering
        \includegraphics[width=0.35\textwidth]{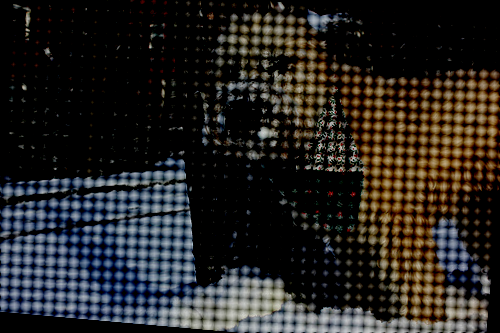}
        \includegraphics[width=0.35\textwidth]{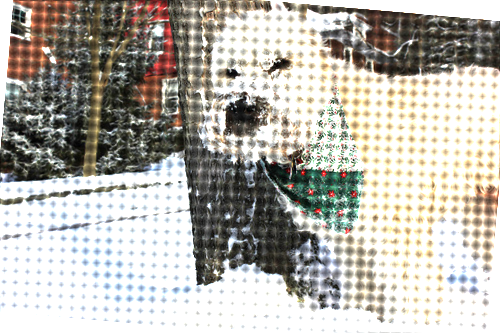}
    \end{subfigure}
    \begin{subfigure}[t]{.30\textwidth}
        \centering
        \includegraphics[width=0.35\textwidth]{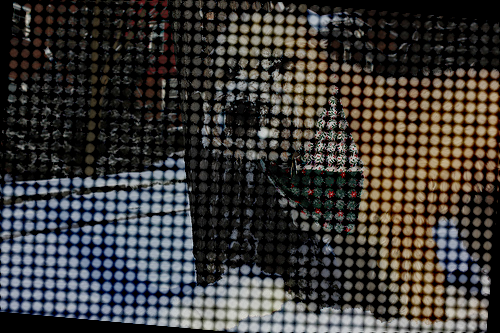}
        \includegraphics[width=0.35\textwidth]{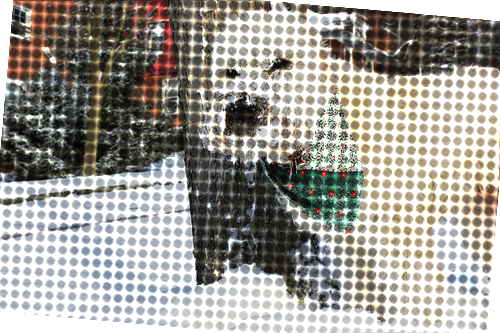}
    \end{subfigure}
    \begin{subfigure}[t]{.15\textwidth}
        \centering
        \includegraphics[width=0.7\textwidth]{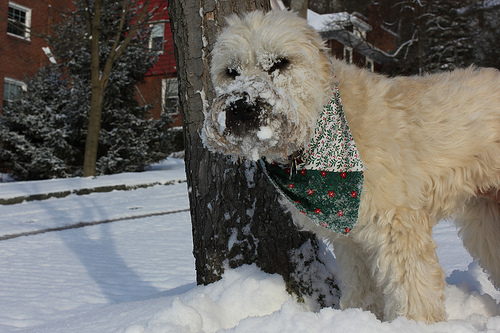}
      \end{subfigure}\\
    \begin{subfigure}[t]{.15\textwidth}
        \centering
        \includegraphics[width=0.7\textwidth]{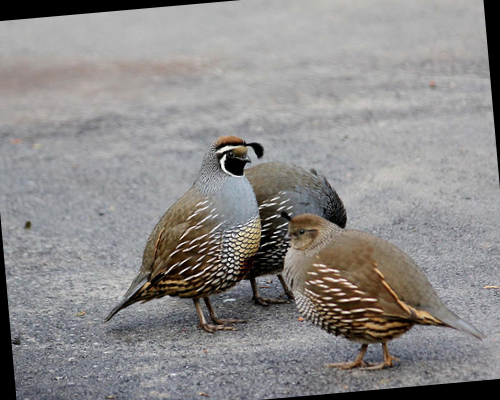}
    \end{subfigure}
    \begin{subfigure}[t]{.30\textwidth}
        \centering
        \includegraphics[width=0.35\textwidth]{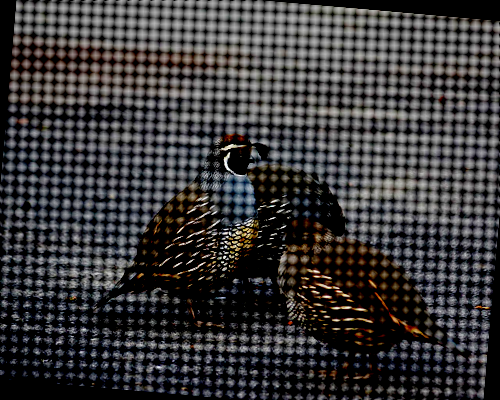}
        \includegraphics[width=0.35\textwidth]{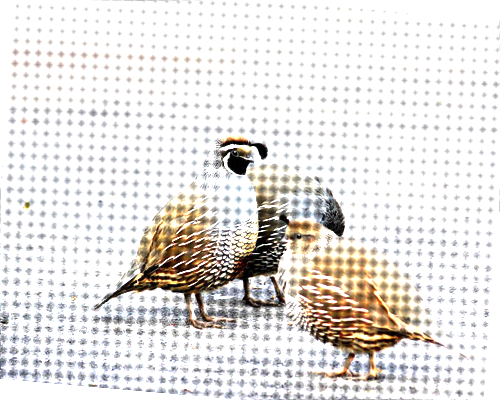}
    \end{subfigure}
    \begin{subfigure}[t]{.30\textwidth}
        \centering
        \includegraphics[width=0.35\textwidth]{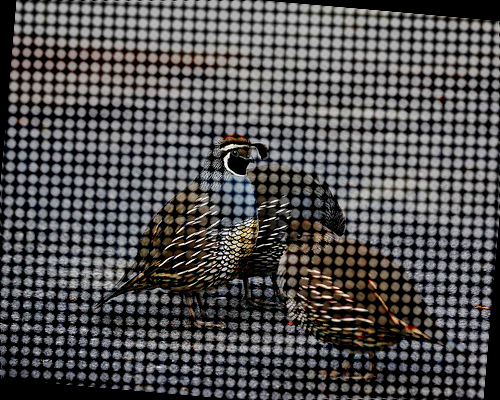}
        \includegraphics[width=0.35\textwidth]{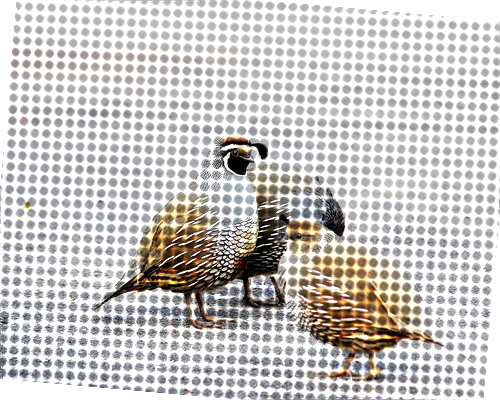}
    \end{subfigure}
    \begin{subfigure}[t]{.15\textwidth}
        \centering
        \includegraphics[width=0.7\textwidth]{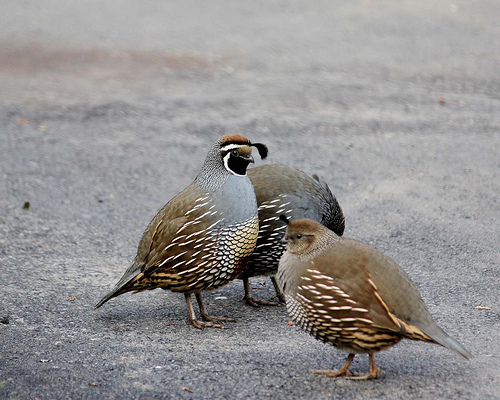}
      \end{subfigure}\\
    \begin{subfigure}[t]{.15\textwidth}
        \centering
        \subcaption{Rotated}
        \label{fig:InverseIN_Calc:rotated}
    \end{subfigure}
    \begin{subfigure}[t]{.30\textwidth}
        \centering
        \subcaption{Inverse}
        \label{fig:InverseIN_Calc:inverse}
    \end{subfigure}
    \begin{subfigure}[t]{.30\textwidth}
        \centering
        \subcaption{$10\times$ refined inverse}
        \label{fig:InverseIN_Calc:refined}
    \end{subfigure}
    \begin{subfigure}[t]{.15\textwidth}
        \centering
        \subcaption{Original}
        \label{fig:InverseIN_Calc:original}
    \end{subfigure}\\          
    
    \caption{Computation of the inverse, analogous to \cref{fig:Inverse_Calc}, for images from ImageNet \citep{ImageNet}.} 
    \label{fig:InverseIN}
\end{figure*}

	\section{Experiment Details}\label{sec:experiment-details} %

\subsection{Details for \cref{sec:methodheuristic}}\label{sec:details-methodheuristic} %
To evaluate \methodHeuristic we use the following classifiers. Note that
\cref{tab:addit-heuristicResults} in \cref{sec:addit-methodheuristic}
contains results for further datasets:

\begin{description}
\item[MNIST \cite{MNIST}] We trained a convolutional network
  consisting of $\text{\textsc{Conv2d}}(k, n)$, with $k \times k$
  filter size, $n$ filter channels and stride 1, batch norm
  $\text{\textsc{BN}}$ \cite{batchNorm}, maximum pooling $\text{\textsc{MaxPool}(k)}$ on $k \times k$ grid,
  $\text{\textsc{Dropout}}(p)$ \cite{dropout}
  with probability $p$ and linear layers $\text{\textsc{Lin}}(a, b)$
  from $\mathbb{R}^{a}$ to $\mathbb{R}^{b}$.

  \begin{align*}
    &\text{\textsc{Conv2d}}(5, 32), \text{\textsc{ReLU}}, \text{\textsc{BN}}\\
    &\text{\textsc{Conv2d}}(5, 32), \text{\textsc{ReLU}}, \text{\textsc{MaxPool}(2)}, \text{\textsc{Dropout}}(0.2)\\
    &\text{\textsc{Conv2d}}(3, 64), \text{\textsc{ReLU}}, \text{\textsc{BN}}\\
    &\text{\textsc{Conv2d}}(3, 64), \text{\textsc{ReLU}}, \text{\textsc{BN}}, \text{\textsc{MaxPool}(2)}, \text{\textsc{Dropout}}(0.2)\\
    &\text{\textsc{Conv2d}}(3, 128), \text{\textsc{ReLU}}, \text{\textsc{BN}}\\
    &\text{\textsc{Conv2d}}(1, 128), \text{\textsc{ReLU}}, \text{\textsc{BN}}, \text{\textsc{Flatten}}\\
    &\text{\textsc{Lin}}(128, 100), \text{\textsc{ReLU}}\\
    &\text{\textsc{Lin}}(100, 10)\\    
  \end{align*}

  We used data normalization for MNIST and trained for 180 epochs with
  SGD, starting from learning rate $0.01$, decreasing it by a factor
  of $10$ every $60$ epochs. No other pre-processing was used.

\item[FashionMNIST \cite{FMNIST}] We trained a ResNet-18 with data
  normalization. We trained for 180 epochs with SGD with an initial
  learning rate of $0.01$, lowering it by a factor of 10 every $60$
  epochs.
  
\item[CIFAR \cite{cifar}] We trained a ResNet-18 with data
  normalization. We trained for 90 epochs with SGD with an initial
  learning rate of $0.1$, lowering it by a factor of 10 every $30$
  epochs. We resized GTSRB images to $32 \times 32 \times 3$.
  
\item[ImageNet \cite{ImageNet}] We used the pre-trained ResNet50 from
  \texttt{torchvision}:
  \url{https://pytorch.org/docs/stable/torchvision/models.html}.
\end{description}

\subsection{Details for \cref{sec:methodglobal}}\label{sec:details-methodglobal} %

In \cref{sec:methodglobal} we use a ResNet-18 architecture for MNIST and a ResNet-110 for CIFAR-10 and, as in, \cref{sec:details-methodheuristic}, ResNet-50 for (R)ImageNet.
We trained them to be
robust to image transformations (rotation, translation) as well as
$\ell^{2}$ noise.

To train networks that perform well when randomized smoothing is applied, we
utilize the training procedure
\textsc{SmoothAdv}\textsubscript{\textsc{PGD}} as outlined in
\citet{Salman}.  For each batch of samples we apply a randomized data
augmentation, vignetting, and Gaussian blur. After
this prepossessing we then apply
\textsc{SmoothAdv}\textsubscript{\textsc{PGD}} (noise restricted to the vingetted area) then evaluate or train on the batch.

The intuition behind the Gaussian blur is that many artifacts, such as
the interpolation error are have high frequencies. The blur acts as a
low-pass filter and discards high frequency noise. This does not
strongly impact the classification accuracy, but drastically reduces
the error estimate and therefore the amount of noise that needs to be
added for robust classification. The filter is parameterized by $\sigma_{b}$ and the filter size $s_{b}$. Formally the filter is a convolution with a filter matrix $F \in \mathbb{R}^{s_b \times s_b}$. Each entry in $F$ is filled with values of a two dimensional Gaussian distribution centered at the center of the matrix and evaluated at the center of the entry. Afterwards the matrix is normalized such that $\sum_{i, j}, F_{i, j} = 1$.

In the error estimation and inference we use the same prepossessing as during training.

\paragraph{MNIST} %
For MNIST we use a ResNet-18 (that takes a single color channel in the
input layer), which we trained with PGD step size $0.2$,
batch size $1024$, and initial learning rate $0.01$ over 180 epochs,
lowering the learning rate every 60 epochs.
For \methodGlobalD we use $\sigma=0.22$ and data augmentation with rotations in $[-90, 90]$ degrees for the rotation model and $\sigma = 0.3$ and random translations of $\pm 50\%$ for the translation model.
For the Gaussian blur we use $\sigma_{b} = 2.0$ with filter size $s_{b} = 5$ on all models.

For \methodGlobalx we use a model trained with  $\sigma=0.15$ for rotations and but the same translation model.

\paragraph{CIFAR-10} %
For \methodGlobalD we train a ResNet-110 with batch size 256, $\sigma=0.25$
and random rotations in $[-60, 60]$ as well as
\textsc{SmoothAdv}\textsubscript{\textsc{PGD}} with $m=1$ samples, $t=8$ steps and warmup of 10 epochs for a perturbation size of $0.5$.
We train over 150 epochs and lower the learning rate every 50 steps.
For the Gaussian blur we use $\sigma_{b} = 1.0$ and $s_{b} = 5$.

For \methodGlobalx we use $\sigma=0.12$, perturbation size of $0.25$, $m=8$ and $t=1$ and keep other parameters the same.
Both variants take about 17 minutes per epoch on a single GeForce RTX 2080 Tis.

\paragraph{(Restricted) ImageNet} %
We trained with a batch size of 400 for 90 epochs using stochastic gradient decent with a learning rate
starting at 0.1, which is decreased by a factor 10 every 30 epochs.
On both datasets , we used $\sigma = 0.5$
and PGD step size $1.0$, as well as $\sigma_{b} = 2.0$ and $s_{b} = 5$.
For Restricted ImageNet we train with
random rotation in $[-60, 60]$ and for ImageNet in $[-30, 30]$.

Training 1 epoch of ImageNet with 6 GeForce RTX 2080 Tis and a 16-core node of
aw Intel(R) Xeon(R) Gold 6242 CPU @ 2.80GHz takes roughly 2.5 hours
and roughly 30 minutes for Restricted ImageNet.
For the accuracy of $b$ in \cref{tab:rotResults}, we evaluate four settings --- with vingetting, with Gaussian blur, with both and with neither --- and report the highest.
\cref{tab:base} shows a comparison across all settings.

\begin{table}[ht]
  \centering
  \caption{Base models evaluated on the whole data set either with Gaussian blur (G), Vingetting (V), both or neither. }
  \label{tab:base}
  \begin{tabular}{@{}lrrrrr@{}}
    \toprule
    Model & $T^{I}$ & Standard & +G & +V  & +G+V\\
    \midrule
    MNIST, \methodGlobalD & $R^{I}$  & 0.98 & 0.98 & 0.98 & 0.98 \\
    MNIST, \methodGlobalx & $R^{I}$  & 0.98 & 0.98 & 0.98 & 0.98 \\
    MNIST & $\Delta^{I}$ & 0.94 &  0.92 & 0.94  & 0.93\\
    CIFAR-10 \methodGlobalD & $R^{I}$ & 0.34 & 0.36 & 0.56  & 0.56\\
    CIFAR-10 \methodGlobalx & $R^{I}$ & 0.75 & 0.76 & 0.70 & 0.70 \\
    RImageNet & $R^{I}$ & 0.77  & 0.77 & 0.78  & 0.77 \\
    ImageNet & $R^{I}$ & 0.38 & 0.32  & 0.38 & 0.32 \\
    \bottomrule
  \end{tabular}
\end{table}

To sample $E$ for (R)ImageNet we use a server with an AMD EPYC 7601 processor with 128 threads.

\cref{tab:rotResultsOld} shows a version \cref{tab:rotResults} in the layout of the version of this paper published at NeurIPS'20.

\begin{table}[h]
  \centering
\begin{small}
  \caption{Evaluation of \methodGlobal for $T^{I} := R^{I}$. $\epsilon_{\max}$ is computed on the training set. We show the test set accuracy of $b$, certified accuracy of $g$ and distribution of the obtained certification radius $r_{\gamma}$, along with the average run time $t$ and the number of used samples $n_{\gamma}, n_{\delta}$. $^{\#}$ denotes values obtained by sampling. Each certificate hold with overall confidence $0.99$.}
    \label{tab:rotResultsOld} %

    \begin{tabular}{@{}lrrrrrrrrrrr@{}} \toprule
    & &                   &          \multicolumn{2}{c}{Acc.}    &       \multicolumn{3}{c}{$r_{\gamma}$ percentile} &  \\
  \cmidrule(lr){4-5} \cmidrule(lr){6-8}
Dataset & \multicolumn{1}{c}{$\epsilon_{\max}$} & \multicolumn{1}{c}{$E$} & \multicolumn{1}{c}{$b$} & \multicolumn{1}{c}{$g$} & \multicolumn{1}{c}{\nth{25}} & \multicolumn{1}{c}{\nth{50}} & \multicolumn{1}{c}{\nth{75}} & \multicolumn{1}{c}{t [s]} & \multicolumn{1}{c}{$n_{\gamma}$} & \multicolumn{1}{c}{$n_{\delta}$} \\
\cmidrule(lr){1-11}
MNIST & 0.36  & 0.45 & 0.98 & 0.89 & 52.95 & 57.22 & 57.22 & 21.56 & 200 & 10000\\
CIFAR-10 & 0.51  & 0.55 & 0.56 & 0.31 & 24.80 &  $30.00^{\dagger}$ & $30.00^{\dagger}$ & 89.75 & 50 & 15000\\
CIFAR-10 & 0.51  & 0.55 & 0.56 & 0.32 &  $30.00^{\dagger}$ &  $30.00^{\dagger}$ & $30.00^{\dagger}$ & 351.47 & 200 & 15000\\
RImageNet & 0.91  & 1.20 & 0.78 & 0.74 & $30.00^{\dagger}$ & $30.00^{\dagger}$ & $30.00^{\dagger}$  & 100.73 & 50 & 2500\\
RImageNet & 0.91  & 1.35 & 0.78 & 0.64 & $30.00^{\dagger}$ & $30.00^{\dagger}$ & $30.00^{\dagger}$  & 100.13 & 50 & 2500\\
ImageNet & 0.91  & 0.95 &  0.38 & 0.30 & 14.51 & 24.34 &   $30.00^{\dagger}$& 100.21 & 50 & 2500\\
ImageNet & 0.91  & 1.20 &  0.38 & 0.23 & 12.38 & 21.47 &   $30.00^{\dagger}$& 100.73 & 50 & 2500\\
ImageNet & 0.91  & 1.35 &  0.38 & 0.16 & 10.46 & 21.47 &   $30.00^{\dagger}$& 100.44 & 50 & 2500\\
\bottomrule
\end{tabular}

\end{small}
\end{table}

\subsection{Details for \cref{sec:methodindividual}}\label{sec:details-methodindividual} %

For rotation we use the $\sigma=0.22$ model as in \cref{sec:details-methodglobal} and for translation also the same model.

	\section{Additional Experiments}\label{sec:addit-exper} %

\subsection{Additional Results for \cref{sec:methodheuristic}}\label{sec:addit-methodheuristic} %
\begin{table*}[h!]
  \centering
\begin{small}    
  \caption{ Extended version of \cref{tab:heuristicResults}.
    Evaluation of \methodHeuristic on 1000 images. The attacker used
     \texttt{worst-of-100}. We use $n_{\gamma} = 1000, \sigma_\gamma =
     \Gamma_{\pm}$.} %
    \label{tab:addit-heuristicResults} %

    \begin{tabular}{@{}llccccc@{}} \toprule
      & & & Acc. &  \multicolumn{2}{c}{adv. Acc.}  & \\
      \cmidrule(lr){4-4} \cmidrule(lr){5-6} 
      Dataset   & $T^{I}$  & $\Gamma_{\pm}$ & $b$ & $b$ & $g$ & t [s]\\
      \cmidrule(lr){1-7}

      MNIST     & $R^{I}$ & $30^{\circ}$ & 0.99 & 0.73 & 0.99 & 0.97\\
      FMNIST & $R^{I}$ & $30^{\circ}$ & 0.91 & 0.13 & 0.87 & 7.98\\
      CIFAR-10     & $R^{I}$ & $30^{\circ}$ & 0.91 & 0.26 & 0.85 & 0.95\\
      GTSRB & $R^{I}$ & $30^{\circ}$ & 0.91 & 0.30 & 0.88 & 8.00\\
      ImageNet  & $R^{I}$ & $30^{\circ}$ & 0.76 & 0.56 & 0.76 & 5.43\\
      \cmidrule(lr){1-7}  
      MNIST     & $\Delta^{I}$ & $4$ & 0.99 & 0.03 & 0.53 & 0.86\\
      FMNIST & $\Delta^{I}$ & $4$ & 0.91 & 0.10 & 0.50 & 6.12 \\
      CIFAR-10     & $\Delta^{I}$ & $4$  & 0.91 & 0.44 & 0.79 & 0.95\\
      GTSRB & $\Delta^{I}$ & $4$ & 0.91  & 0.30 & 0.63 & 5.17 \\
      ImageNet  & $\Delta^{I}$ & $20$ & 0.76 & 0.65 & 0.75 & 6.70\\
      \bottomrule
    \end{tabular}

\end{small}   
\end{table*}

\cref{tab:addit-heuristicResults} is an extended version of
\cref{tab:heuristicResults} and provides results for additional
datasets.

\begin{table*}[ht!]
  \centering
\begin{small}    
  \caption{We first use \methodHeuristic to obtain the certification
    radius $r_{\gamma}$ on 30 images and subsequently sample from the
    parameter space indicated by $\Gamma_{\pm} = r_{\gamma}$ and
    checked whether the certificate holds for them. We use 30 samples
    and $n_{\gamma} = 2000$ samples for the smoothed classifier. The
    last column shows the number of images for which we found
    violations. }
    \label{tab:heuristicViolations} %

    \begin{tabular}{@{}llccc@{}} \toprule
      Dataset   & $T^I$ & $\Gamma_\pm$ & median $r_{\gamma}$ & $r_{\gamma}$ violated\\
      \midrule
      MNIST     & $R^{I}$ & $30^{\circ}$ & 28.34 & 0\\
      FMNIST     & $R^{I}$ & $30^{\circ}$ & 13.45 & 1\\
      CIFAR-10     & $R^{I}$ & $30^{\circ}$ & 19.16 & 14\\
      GTSRB     & $R^{I}$ & $30^{\circ}$ & 20.93 & 0\\
      ImageNet     & $R^{I}$ & $10^{\circ}$ & 27.13 & 1\\

      \midrule

      MNIST     & $\Delta^{I}$ & $4$ & 1.12 & 0\\
      FMNIST     & $\Delta^{I}$ & $4$ & 1.78 & 1\\
      CIFAR-10    & $\Delta^{I}$ & $4$ & 4.76 & 14\\
      GTSRB     & $\Delta^{I}$ & $4$ & 2.58 & 0\\
      ImageNet     & $\Delta^{I}$ & $20$ & 16.43 & 0\\
      \bottomrule
    \end{tabular}
\end{small}   
\end{table*}

\begin{table*}[ht!]
  \centering
\begin{small}    
  \caption{Same setup as in \cref{tab:heuristicViolations}, but with circular vignetting.}
    \label{tab:heuristicViolationsVingetteInterpolation} %

    \begin{tabular}{@{}llcccc@{}} \toprule
      Dataset   & $T^I$ & $\Gamma_\pm$ & median $r_{\gamma}$ & $r_{\gamma}$ violated & $r_{\gamma}$ violated, no interpolation\\
      \midrule
      MNIST     & $R^{I}$ & $30^{\circ}$ & 28.34  & 0 & 0\\
      FMNIST     & $R^{I}$ & $30^{\circ}$ &  17.07 & 0 & 0 \\
      CIFAR-10     & $R^{I}$ & $30^{\circ}$ & 11.49  & 10 & 0\\
      GTSRB     & $R^{I}$ & $30^{\circ}$ & 25.28 & 0 & 0\\

      \bottomrule
    \end{tabular}
\end{small}   
\end{table*}

\subsection{``Certification Radius'' of \methodHeuristic}\label{sec:methodheuristic-certificate} %
As \methodHeuristic uses \cref{thm:bound} to justify the heuristic,
this also makes it tempting to use the bound $r_{\gamma}$ provided by
it. However, as the assumptions of \cref{thm:bound} are violated it
does not formally present a certification radius. Here we investigate
if and how much it holds nevertheless.  To do this we construct a
smoothed classifier $g$ from an undefended base classifier $b$ and
calculated the certification radius $r_{\gamma}$. Subsequently, we
sampled 100 new rotated images in the parameter space induced by
$\Gamma_{\pm} = r_{\gamma}$ and evaluated on them. The results are shown in
\cref{tab:heuristicViolations}.  While generally robust, the radius
does not constitute a certificate, as we can clearly find violations.

In the context of rotation $R^{I}$ we add circular vignetting (as we
do for \methodGlobal and \methodIndividual) to make the behavior
closer to a composing transformation.  For this experiment, we retrained
the same networks, but applied the vignette during training.  Results
are shown in \cref{tab:heuristicViolationsVingetteInterpolation} where
we can see that this already decreases the number of violations for
CIFAR-10 and FMNIST.  In a final step we assume knowledge of the attacker
parameter $\gamma$ and replace $R^{I}_{\beta} \circ R^{I}_{\gamma}$
(for the same images) with $R^{I}_{\beta+\gamma} $ in the evaluation
of the classifier, in which case \cref{thm:bound} should hold and
indeed we don't observe any more violations.

\subsection{Additional Results for \cref{sec:methodglobal}}\label{sec:addit-methodglobal} %

\paragraph{Beyond Bilinear Interpolation}
\methodHeuristic and \methodGlobal can directly be applied to image
transformations using other interpolation schemes without any
adaption. \methodIndividual, however, requires the adaption of the
inverse algorithm. While this is generally possible, we consider it
beyond the scope of this work.

We guess (based on $1000$ samples) and verify $E$ using $1000$ samples for $\vx$ and $8000$ for $\beta$.
We summarize these results in \cref{tab:E_bicubic}.
On datasets other than MNIST we observe $E$ larger than possible. At manageable levels, the $q_{E}$ becomes too low for practical purposes.

On MNIST with the same settings as for \methodGlobalD we certify 90 out of 100 images at $r_{\gamma}=30$ (for bilinear interploation with $E=0.45$ 91 can be certified).

\begin{table*}[ht!]
  \centering
  \caption{$E$ and $q_{E}$ for bicubic interpolation.}
  \label{tab:E_bicubic}
  \begin{tabular}[t]{@{}lrr@{}}
    \toprule
    Dataset & $E$ & $q_{E}$\\
    \midrule
    MNIST & 0.5 & 0.99 \\
    CIFAR-10 & 1.10 & 0.99 \\
    CIFAR-10 & 0.55 & 0.27 \\
    ImageNet & 2.50 & 0.99 \\
    ImageNet & 1.20 & 0.28 \\
    \bottomrule
  \end{tabular}
\end{table*}

\subsection{Audio Volume Change}\label{sec:audio-volume-change} %
To show that our method can be used beyond image transformation we
showcase an adaption to audio volume changes.
The volume of an audio signal can be changed by multiplying the signal with a
constant. In order to change the signal $\vx$ by $\beta$ (measured in decibel
$[\beta] = \text{dB}$) we multiply $\vx$ by $10^{\beta/20}$. Thus the
transformation is $\psi_\beta(\vx) := 10^{\beta/20} \cdot \vx$, which composes:
\begin{equation*}
  \psi_\beta \circ \psi_\gamma(\vx)
  = 10^{(\beta + \gamma)/20} \cdot \vx
  = \psi_{\beta + \gamma}(\vx).
\end{equation*}

In practice such signals are stored in final precision, e.g. 16-bit, thus
potentially introducing rounding errors, with an $\ell^2$-norm bound by $E$. If
this is ignored \methodHeuristic can be applied to obtain guarantees. Otherwise,
\methodGlobal and \methodIndividual can be used to obtain sound bounds.

To evaluate this we use the speech commands dataset \cite{gcommands}, consisting
of 30 different commands, spoken by people, which are to be classified. The
length of the recordings are one second each. 
We use a classification pipeline that converts audio wave forms into
MFCC spectra \cite{MFCC} and then treats these as images and applies
normal image classification.  We use a ResNet-50, that was trained with
Gaussian noise, but not
\textsc{SmoothAdv}\textsubscript{\textsc{PGD}}.  We apply the noise
before the waveform is converted to the MFCC spectrum.

For \methodGlobal we estimate $E$ to be $0.005$ with the parameters
$\rho_{E} = 0.05, \sigma_{\gamma} = 3$ and $\Gamma = 3$ (for which $q_{E} \approx 0.75$). On 100
samples, the base classifier $f$ was correct $93$ times.
At $r_{gamma}$ of 1, 2, 3 and 4 the certified accuracy was 0.92, 0.89, 0.83 and 0.69 respectively.
This corresponds to $\pm 1.12$, $\pm 1.26$, $\pm 1.41$ and $\pm 1.58$ dB.
At $n_{\gamma} = 150$ and $n_{\epsilon} = 400$ the average certification
time was $26.80$ s. We use $\alpha_{\gamma} = 0.004, \alpha_{\epsilon} = \tfrac{0.005}{n_{\gamma}}$, assuming (but not computing) $\alpha_{E} = 0.001$ here, for a total confidence of $0.99$ in each certificate.

To investigate \methodIndividual we use
$\sigma_{\gamma} = 0.85, \Gamma = 1.05$. For 92 out of $100$ perturbed
audio signals to compute $\epsilon$.  We obtained
$\epsilon_{\max} \leq 0.0055$ and for 68 an $\epsilon \leq 0.005$,
which together with our results for \methodGlobal suggests the
applicability of the method.
For each signal we used $100$ samples for
$\beta$.  For cases with $\epsilon_{\max} > 0.0055$ we in fact observed
$\epsilon_{\max} \gg 0.0055$, as here many parts of the signal were amplified beyond
the precision of the 16-bit representation and clipped to $\pm
1$. This makes the information unrecoverable and sound error bound
estimates large.

  \begin{table*}[ht]
\begin{minipage}[t]{.45\textwidth}
  \centering
  \caption{Maximum observed errors and without gaussian blur (G) and without vignetting (V).}
  \label{tab:errablation}  
  \begin{tabular}{@{} l r r r r @{}}
    \toprule
    Dataset   & Both & -V & -G & -V-G \\
    \midrule
    MNIST     & \textbf{0.36}  &   \textbf{0.36}  &   2.47  &   2.51    \\
    CIFAR-10   & \textbf{0.51}  &   6.08  &   2.66  &   18.17   \\
    ImageNet  & \textbf{0.91}  &   70.66 &   9.25  &   75.69   \\
    \bottomrule
  \end{tabular}
\end{minipage}
\hfill
\begin{minipage}[t]{.45\textwidth}
  \caption{Correct classifications and by the model and verifications
    by DeepG \cite{DeepG}, with and without vignetting (V), out of 100
    images.}
  \label{tab:deepg}    
  \begin{tabular}{@{} l r r r @{}}  
    \toprule
    Model & Correct & \cite{DeepG}  & \cite{DeepG}+V    \\
    \midrule
    MNIST   & 98      &  86   & 87    \\
    CIFAR-10 & 74      &  65   & 32    \\
    CIFAR-10+V & 78      & 63    & 23  \\
    \bottomrule
  \end{tabular}
\end{minipage}
\end{table*}

\section{Further Comparison and Ablation}\label{sec:furth-comp-ablat} %
To show that the vignette and Gaussian blur are essential to our
algorithm we perform a small ablation study. \cref{tab:errablation}
shows the maximal error observed when sampling as in \methodGlobal.
We use the same setup as in \cref{sec:methodglobal}, but with 10000
samples for ImageNet.

Both, vignetting and Gaussian blur reduce the error bound
significantly for \methodGlobal and \methodIndividual. On CIFAR-10 and
ImageNet vignetting is very impactful because the corners of images
are rarely black in contrast to MNIST.  \citet{LinyiLiSmoothing} uses
vignetting for the same reason.  Without either of the methods
bounding the error would not be feasible.

For \methodIndividual vignetting is crucial, even for MNIST, as we can
make no assumptions for parts that are rotated into the image. Thus we
need to set these pixels to the full $[0,1]$ interval (see
\cref{fig:Inverse_Calc}). Without Gaussian blur the certification
rate drops to $0.11$.

Further, we extend this comparison to related work: We extended
\citet{DeepG} (Table~1 in their paper) to include vignetting. The
results are shown in \cref{tab:deepg}.  We also retrained their
CIFAR-10 model with vignetting (CIFAR-10+V) for completeness. While
vignetting on MNIST slightly helps (+1 image verified) on CIFAR-10 it
leads to a significant drop.  Including Gaussian blur into [11] would
require non-trivial adaption of the method. However, we implemented
this for interval analysis (on which their method is built) and found
no impact on results.

}{}

\message{^^JLASTPAGE \thepage^^J}

\end{document}